\documentclass{elsarticle}

\usepackage{lineno,hyperref}
\modulolinenumbers[5]
\usepackage{comment}
\usepackage{amssymb}
\usepackage{graphicx}
\usepackage{graphicx}
\usepackage{amsfonts}
\usepackage{amsmath}
\usepackage{multirow}
\usepackage{pgf,tikz}
\usepackage{algorithmic}
\usepackage{algorithm}
\usepackage{color}
\usepackage{subcaption}
\usepackage{booktabs}
\usepackage{appendix}
\newcounter{proposition}

\usepackage{soul}

\newcounter{theorem}

\newcounter{corollary}

\newcounter{remark}

\newcounter{proof}
\newenvironment{proof}{\refstepcounter{proof}\par\medskip
        \noindent\textit{Proof}\ }{\par\medskip}
\newenvironment{pf}{\begin{proof}}{\end{proof}}

\newcounter{definition}

\newcommand{\definition}[1][]{%
  \refstepcounter{definition}%
  \par\medskip
  \noindent\textbf{Definition~\arabic{definition}%
  \if\relax\detokenize{#1}\relax\else\ \textbf{(#1)}\fi.}\ %
}

\newcounter{example}

\newcommand{\example}[1][]{%
  \refstepcounter{example}%
  \par\medskip
  \noindent\textbf{Example~\arabic{example}%
  \if\relax\detokenize{#1}\relax\else\ \textbf{(#1)}\fi.}\ %
}

\newenvironment{highlights}{\section*{Highlights}\begin{itemize}}{\end{itemize}}

\journal{Fuzzy Sets and Systems}

\begin{document}

\begin{frontmatter}

\title{
  Matrix Aggregation Operators
  \footnote{This research has been partially funded by the Government of Spain, Grant Plan Nacional de I+D+i PID2024-155289NB-I00 funded by MICIU/AEI/10.13039/501100011033 and ERDF/EU and  PID2025-168223NB-I00, funded by MCIN/AEI/10.13039/501100011033/ERDF UE.}}

\address[A]{{Faculty of Statistical Studies, Complutense University of Madrid}}
            
\address[B]{{Department of Statistics, Computer Science and Mathematics, Universidad Pública de Navarra}}
\address[C]{{Faculty of Mathematics, Complutense University of Madrid}}
\address[D]{Instituto Universitario de Estadística y Ciencia de Datos, Complutense University of Madrid}
\address[E]{Instituto de Matemática Interdisciplinar, Complutense University of Madrid}
\address[F]{{Institute of Smart Cities (ISC), Universidad Pública de Navarra}}

\author[A,D]{Inmaculada Guti\'errez}\ead{inmaguti@ucm.es}

\author[B,F]{{Asier Urio-Larrea}}\ead{asier.urio@unavarra.es}

\author[C,E]{{J.Tinguaro Rodríguez$^{*}$}}\ead{jtrodrig@ucm.es}\cortext[cor1]{Corresponding author: J.Tinguaro Rodríguez, jtrodrig@ucm.es}

\author[A,D]{{Daniel G\'omez}}\ead{dagomez@estad.ucm.es}

\author[C,E]{{Javier Montero}}\ead{monty@mat.ucm.es}

\author[B,F]{{Humberto Bustince}}\ead{bustince@unavarra.es}

\begin{abstract}
Aggregation theory has traditionally focused on operators defined over vectors. However, many applications—including Multi-Criteria Decision Making, Group Decision Making, Fuzzy Rule-Based Classification Systems, and overlap/grouping indices—require aggregating information naturally structured as a matrix of membership degrees (e.g., where a set of objects interacts with a family of fuzzy sets). Despite this, no formal framework has been proposed for this class of operators, partly due to the common practice of flattening matrices into vectors (which discards structural information) and partly due to a reliance on decomposable operators that aggregate rows and columns sequentially. This paper addresses this gap by formalizing the notion of a matrix aggregation operator (MAO). We analyze the decomposability and symmetry properties of MAOs, showing that certain operators cannot be expressed in decomposable form and examining several notions of symmetry. Finally, we introduce a family of MAOs termed maximum entropy global coverage indices (MEGCIs), provide a construction method for them based on combining grouping functions and MEOWA operators, and illustrate their usefulness in cluster quality assessment through an extensive computational study. 
\end{abstract}


\begin{highlights}
\item Definition of a new family of aggregation operators: Matrix Aggregation Operators (MAOs).
\item Study of symmetry and decomposability properties of MAOs.
\item Several examples of MAOs in applications.
\item Introduction of a particular class of MAOs, termed Maximum Entropy Global Coverage Indexes (MEGCIs).
\item Application of MEGCIs for cluster cohesion and quality assessment.
\end{highlights}
\end{frontmatter}


\section{Introduction}\label{sec:intro}

\noindent Aggregation theory constitute a fundamental and highly active research area in fuzzy logic and soft computing, as it deals with the crucial step of synthesizing or summarizing uncertain information. The origins of aggregation can be found in Zadeh's seminal paper \cite{Zadeh1965}, where he proposed that the membership degree of an element $x$ to the union or intersection of two fuzzy sets, $A_1$ and $A_2$, can be seamlessly derived by either computing the maximum or minimum of the respective membership values. However, in his original formulation, Zadeh implicitly distinguished between two fundamentally different aggregation scenarios, a distinction explicitly raised in \cite{MonteroComprFuzz86,MonteroExtFuzz87}.

In the first context, a single object $x$ and $K$ different fuzzy sets $A_1, \dots, A_K$ are considered, and the focus is placed on the aggregation of the membership degrees $\mu_{A_1}(x), \dots, \mu_{A_K}(x)\in[0,1]$ associated with these fuzzy sets. This situation arises, for instance, when modelling composite concepts such as \textit{tall and rich} or \textit{tall or rich}, where the aggregation is performed across different properties of the same object. Conjunctive and disjunctive aggregation operators naturally appear in these settings associated with determining the membership of an object $x$ to a set-theoretic operation consisisting of intersections and unions of fuzzy sets, such as $\mu_{A_1 \cap A_2}(x)$,  $\mu_{A_1 \cup A_2}(x)$, or $\mu_{A_1 \cap (A_2 \cup A_3)}(x)$, based on the individual membership values $(\mu_{A_1}(x), \mu_{A_2}(x),\mu_{A_3}(x))$. 

The second scenario concerns the aggregation of information from several objects with respect to a common property. Let $x_1,\dots,x_N$ be a group of objects and let $A$ be a fuzzy set, such that $\mu_A(x_1),\dots,\mu_A(x_N) \in [0,1]$ quantify the degree to which each of these objects belong to $A$. The question then arises as to which degree the group of objects as a whole can be asserted to verify $A$. For instance, if $x_1=John,\ x_2=Mary,\ x_3=Peter$ and $A=tall$, to which degree is this group of persons \textit{tall}? Although conjunctive and disjunctive operators can again be used in this context to assess whether all or any of these persons are tall, it also makes sense to use averaging operators to quantify the global verification of the property by the group.

In many practical situations, however, both dimensions coexist simultaneously. In these cases, a collection of objects $x_1,\dots,x_N$ interacts with a family of fuzzy sets $A_1,\dots,A_K$, and as shown in Table \ref{tab1} the available information is naturally represented by a matrix $\mathbf{U}=(u_{ij})\in\mathcal{M}_{N,K}([0,1])$ of membership degrees $u_{ij}=\mu_{A_j}(x_i)\in[0,1]$, where $i=1,\ldots,N$, $j=1,\ldots,K$, and $\mathcal{M}_{N,K}([0,1])$ denotes the set of all $N\times K$ matrices whose elements belong to $[0,1]$. Once in this matricial framework, diverse purposes may require aggregating $\mathbf{U}$ into a representative value $F(\mathbf{U})\in[0,1]$.

\begin{table}[h!]
\centering
\begin{tabular}{c|cccc}
 & $A_1$ & $A_2$ & $\cdots$ & $A_K$ \\
\hline
$x_1$ & $\mu_{A_1}(x_1)$ & $\mu_{A_2}(x_1)$ & $\cdots$ & $\mu_{A_K}(x_1)$ \\
$x_2$ & $\mu_{A_1}(x_2)$ & $\mu_{A_2}(x_2)$ & $\cdots$ & $\mu_{A_K}(x_2)$ \\
$\vdots$ & $\vdots$ & $\vdots$ & $\ddots$ & $\vdots$ \\
$x_N$ & $\mu_{A_1}(x_N)$ & $\mu_{A_2}(x_N)$ & $\cdots$ & $\mu_{A_K}(x_N)$ \\
\end{tabular}
\caption{$u_{ij}=\mu_{A_j}(x_i)$ membership degree for $N$ objects and $K$ classes.}\label{tab1}
\end{table}

Some instances of this matricial setting are the following: in Multi-Criteria Decision Making (MCDM) problems \cite{Fodor1994,Amo2004}, a set of alternatives exhibits a degree of satisfaction with respect to multiple criteria. A similar or possibly more complex matrix structure is evident in Group Decision Making (GDM), now involving multiple decision-makers \cite{xia2020group,cabrerizo2014building}. In Fuzzy Rule-Based Classification Systems (FRBCSs) \cite{elkano2018chi}, a matrix of association degrees assesses the evidence provided by each rule for each class. Matrix-based aggregation has been prominently featured in the development of  overlap and grouping indices \cite{garcia2014overlap,dimuro2024grouping}. These aggregate the matrix $\mathbf{U}$ into a single degree, respectively quantifying the degree of overlap of $K$ fuzzy sets defined on a finite universe of discourse with $N$ objects, and the degree to which the union of the fuzzy sets covers the universe.

Thus, despite the multitude of scenarios that naturally involve matrices of membership degrees requiring aggregation, to the extent of our knowledge no attempt has been made to endow this specific class of aggregation operators with a formal structure and well-defined properties. The reasons for this may be two-fold. Firstly, any matrix $\mathbf{U}\in\mathcal{M}_{N,K}([0,1])$ can be flattened into a vector $\mathbf{u} \in [0,1]^{N\cdot K}$, and aggregated in that form. This works when the 2D structure of the matrix no longer matters, such as when applying a symmetric pooling operator to a pooling window of a digital image (see e.g. \cite{OWAPooling}). However, in most applications the matrix structure stores by itself valuable information that should be explicitly preserved during the aggregation process, such as semantically meaningful relationships between rows and columns, and which would be destroyed by flattening. Secondly, possibly a more relevant reason is that in many applications it makes sense to aggregate a matrix using what we shall call \textit{decomposable} operators, that is, operators defined through a two-stage procedure, in which either row or column vectors are first aggregated with a certain vectorial operator, and the resulting vector is then aggregated with another such operator. Nevertheless, although this is an extended procedure, in this work it will be shown that general matrix aggregation operators (e.g., the polarization index introduced in \cite{guevara20}) can not be represented in this decomposable form. 

Motivated by these observations and the aforementioned gap in the formal analysis of matrix aggregations, the present work is devoted to the study of this class of operators from both a theoretical and an applied perspective. We first introduce and formalize the notion of matrix aggregation operator (MAO), which encompasses all those scenarios in which the set of membership degrees to be aggregated into a single degree possesses an inherent matrix structure. We then study the decomposability and symmetry of MAOs, showing that certain MAOs cannot be expressed in decomposable form, and examine several row- and column-based symmetry properties together with the relationships between them. Finally, we introduce a specific family of MAOs, termed maximum entropy global coverage indices (MEGCIs), whose applicability to clustering quality assessment is illustrated by means of an extensive computational study.

The remainder of this paper is organized as follows. Some preliminary concepts are presented in Section~\ref{sec:prel}. The definitions of MAO and decomposability are addressed in Section~\ref{sec:MAOs}. Symmetry properties of MAOs are discussed in Section~\ref{sec:symmetry}. Some examples of MAOs relevant for applications are reviewed in Section \ref{sec:examples}. MEGCI operators are studied in Section~\ref{sec:covering}, and its application to clustering quality assessment is analyzed in Section~\ref{sec:application}. Section~\ref{sec:conclusion} provides conclusions and final remarks.

\section{Preliminaries}\label{sec:prel}

This section recalls some concepts that will be useful throughout this work. The main notion is that of aggregation operator (AO), considering a fixed cardinality $n$ of the data to be aggregated (see \cite{Montero2018,GRABISCH20111,Cutello1999} for other approaches based on recursiveness and computability).

\begin{definition}[Aggregation Operator \cite{Dubois1984, Dubois1985, Calvo2002}]\label{def:AggOp}
Let $n\geq2$. An aggregation operator (of dimension $n$) is a function $AO:[0,1]^n \rightarrow [0,1]$  that satisfies the following conditions:

\begin{enumerate}
    \item \textbf{Boundary conditions}. $AO(0,... , 0) = 0$ and $AO(1,... , 1) = 1$.
    \item \textbf{Monotonocity}. For any pair of $n$-tuples $\mathbf{x}=(x_1, \dots, x_n),\ \mathbf{y}=(y_1, \dots, y_n)$ $\in [0, 1]^n$ such that $x_i \leq y_i$ $\forall i \in \{1, \dots, n\}$, it holds that
$AO(\mathbf{x}) \leq AO(\mathbf{y})$; that is, $AO$ is monotonically non-decreasing with respect to each of its arguments.

 \end{enumerate}
 \end{definition}

Additional properties can be imposed on AOs depending on the application. A relevant instance is symmetry, which ensures that the aggregation result is invariant under permutations $\sigma \in \mathcal{S}_n$ of the inputs, where $\mathcal{S}_n$ denotes the set of all permutations of $n$ elements, that is, the symmetric group of degree $n$. Specifically, an aggregation operator is said to be symmetric if $AO(x_1, \dots, x_n) = AO(x_{\sigma(1)}, \dots, x_{\sigma(n)})$ for every permutation $\sigma \in \mathcal{S}_n$.

A first instance of aggregation operator used in this work are $n$-dimensional grouping functions. Grouping functions were first proposed in \cite{Bustince2012a} in a bivariate setting as the disjunctive dual of the conjunction-oriented overlap functions \cite{Bustince2010}. Thus, their extension to an $n$-dimensional context enable computing disjunctions of several fuzzy degrees.

\begin{definition}[$n$-dimensional grouping function \cite{Gomez2016}] The mapping $G_G : [0,1]^n \rightarrow [0,1]$ is said to be a $n$-dimensional grouping function if and only if the following conditions hold.
\begin{enumerate}
\item[$(G_G1)$] $G_G$ is symmetric;
\item[$(G_G2)$] $G_G(\mathbf{x}) = 0$  if and only if $x_i=0$ $\forall i \in \{1,\dots , n\}$;
\item[$(G_G3)$] $G_G(\mathbf{x}) = 1$  if and only $\exists i \in \{1, \dots , n\}$ with $x_i=1$;
\item[$(G_G4)$] $G_G$ is non-decreasing;
\item[$(G_G5)$]$G_G$ is continuous.
\end{enumerate}
\end{definition}

The second instance of aggregation operator used in this work are ordered weighted averaging (OWA) operators. Proposed in \cite{owaYagger}, OWA operators enable representing a wide spectrum of averages, ranging from the minimum to the maximum of the data, by applying a weighting vector to the ordered sample. A weighting vector is a vector $\textbf{w}=(w_1,\dots,w_n) \in [0,1]^n$ such that $\sum_{i=1}^nw_i=1$.

\begin{definition}[Ordered Weighted Averaging (OWA) operators  \cite{owaYagger}] The OWA operator with weighting vector $\textbf{w}$ is a mapping $OWA_{\textbf{w}}:[0,1]^n \rightarrow [0,1]$ computed for any $\textbf{x}\in [0,1]^n$
as $OWA_{\textbf{w}}(\mathbf{x})=\sum_i^nw_ix_{(i)}$, where $x_{(i)}$ denotes the $i$-th largest element of $\textbf{x}$. 
\end{definition}

Note that OWA operators are symmetric, since the data to be averaged is sorted previously to the application of the weighting vector, and thus the aggregation result does not depend on the order in which the data is presented. Moreover, the closeness of an OWA operator to the limiting disjunctive behavior of the maximum (or inversely, to the conjunctive minimum) can be quantified through the orness degree $\mathcal{O}(\mathbf{w})\in[0,1]$ of its weighting vector $\mathbf{w}$, defined as
\begin{equation}\label{eq:orness}
\mathcal{O}(\mathbf{w}) = \frac{1}{n-1} \sum_{i=1}^{n} (n-i) w_i.
\end{equation}
For instance, Eq. \eqref{eq:orness} assigns orness degree 1 to the weighting vector of the maximum operator, 0 to that of the minimum, and 0.5 to the arithmetic mean.

Finally, a subject of extensive research regarding OWA operators has been the determination of adequate or optimal weighting vectors \cite{ohagan1988}. An approach that have found some success is that of using the maximum entropy weights corresponding to a desired or predefined orness degree $d\in[0,1]$ \cite{ohagan1990, FULLER2001:MEOWA, Harmati2022}. This entails obtaining the weighting vector $\mathbf{w}$ with minimum Shannon entropy     
\begin{equation}\label{eq:meowa}
H(\mathbf{w}) = - \sum_{i=1}^{n} w_i \ln w_i,
\end{equation}
subject to $
\mathcal{O}(\mathbf{w})={d}$. As shown in \cite{FULLER2001:MEOWA,Harmati2022}, this optimization problem has a unique solution $\mathbf{w}^d$ for each ${d}\in[0,1]$. Also, it holds that $w^d_i>0$ for $i=1,\ldots,n$ when ${d}\in(0,1)$. Due to their definition, $\mathbf{w}^d$ maximizes the spread of the weights along the $n$ positions, distributing importance across all objects being aggregated as uniformly as possible for the given orness degree $d$. The OWA operator using such a maximum-entropy weighting vector $\mathbf{w}^d$ is usually referred to as the maximum-entropy OWA (i.e, MEOWA) operator for orness degree ${d}$.  

As a final comment, let us remark that, although all the operators are defined on the interval $[0,1]$, any other interval $I\subset\mathbb{R}$ may also be considered.

\section{Matrix Aggregation Operators and Decomposability}\label{sec:MAOs}
\noindent In this section, the notion of Matrix Aggregation Operator (MAO) is formally introduced and illustrated. Subsequently, the relationship between MAOs and
the concept of decomposability is analyzed.  

\subsection{Definition of Matrix Aggregation Operator}

\noindent As discussed in the introduction, many  aggregation problems naturally deals with information arranged in a matrix structure, where rows and columns represent different semantic dimensions of the problem. In such situations, reducing the available information to a simple vector may ignore relevant relationships and interactions associated with this two-dimensional organization. This observation motivates the introduction of aggregation operators specifically designed to process matrices while preserving their inherent structure.

Therefore, let $\mathcal{M}_{N,K}([0,1])$ denote the set of all matrices with $N$ rows and $K$ columns, whose elements belong to the $[0,1]$, interval. A  matrix aggregation operator is a function $F:\mathcal{M}_{N,K}([0,1])\to [0,1]$ that takes as input an $N\times K$ matrix of [0,1] degrees and produces a single degree in [0,1], and satisfies the usual boundary and monotonicity conditions of an aggregation operator. Formally: 

\begin{definition}[Matrix Aggregation Operator ({MAO})]\label{def:mao} 
Let $N,K$ be positive integers such that $N\cdot K \geq 2$. A \emph{matrix aggregation operator} (of dimension $N$ $\times$ $K$) is a mapping
\begin{equation}\label{eq:MAO}
F : \mathcal{M}_{N,K}([0,1]) \longrightarrow [0,1]
\end{equation}
that satisfies the following conditions:
\begin{enumerate}
    \item  $F(\mathbf{0}_{N,K})= 0$, 
    \item  $F(\mathbf{1}_{N,K})=1$,
    \item \textbf{Monotonicity:} For any $\mathbf{U}, \mathbf{U}' \in \mathcal{M}_{N,K}([0,1])$, 
    if $\mathbf{U} \leq \mathbf{U}'$ entrywise (i.e., $u_{ij} \leq u'_{ij}$ for all $i,j$), then $F(\mathbf{U}) \leq F(\mathbf{U}')$,
\end{enumerate}
    where $\mathbf{0}_{N,K}$ (resp. $\mathbf{1}_{N,K}$) denotes the $N \times K$ matrix with all entries equal to $0$ (resp. $1$).
\end{definition}

\begin{remark}\label{rem_maoiffao}
  Clearly, there is a correspondence between MAOs of dimension $N\times K$ and AOs of dimension $N\cdot K$. To see it, note that any one-to-one mapping $f:\nobreak\{1,\ldots,N\}\times\{1,\ldots,K\}\longrightarrow\{1,\ldots,N\cdot K\}$ can be understood as a flattening or vectorization (see e.g. \cite{Hackbusch2012}) operator $f:\mathcal{M}_{N,K}([0,1]) \longrightarrow [0,1]^{N\cdot K}$ such that $f(\mathbf{U})=(u_{f^{-1}(1)},\dots,u_{f^{-1}(N\cdot K)})$ for any $\mathbf{U}=(u_{ij})\in\mathcal{M}_{N,K}([0,1])$. The $f$-flattened version $F_f$ of a MAO $F$ can then be defined as $F_f(\mathbf{v})=F(f^{-1}(\mathbf{v}))$ for any $\mathbf{v}\in[0,1]^{N\cdot K}$. For any flattening operator $f$, it is then straightforward to check that $F$ is a MAO if and only if $F_f$ is an AO. Moreover, when $f$ is fixed, the mapping that assigns to each $N\times K$ MAO $F$ the corresponding flattened AO $F_f$ of dimension $N\cdot K$ defines a one-to-one correspondence. However, given a MAO $F$, its flattened version $F_f$ depends on $f$, meaning that in general $F_f$ is a different AO for each $f$.
\end{remark}

\begin{example}\label{ex:trace} Let $\mathbf{U} \in \mathcal{M}_{N,N}([0,1])$ be a square matrix, and let the mapping $\mathrm{trace}:\mathcal{M}_{N,N}([0,1]) \longrightarrow [0,1]$ be given by
$$
{\mathrm{trace}}(\mathbf{U}) = \frac{1}{N} \sum_{i=1}^{N} u_{ii}.
$$
It is straightforward to see that ${\mathrm{trace}}$ is a MAO:
\begin{enumerate}
\item $
    {\mathrm{trace}}(\mathbf{0}_{N,N}) = \frac{1}{N} \sum_{i=1}^N 0 = 0$.
    \item ${\mathrm{trace}}(\mathbf{1}_{N,N}) = \frac{1}{N} \sum_{i=1}^N 1 = 1$.
    \item If $\mathbf{U} \leq \mathbf{U}'$ entrywise, then $u_{ii} \leq u'_{ii}$ for each $i$, and therefore ${\mathrm{trace}}(\mathbf{U}) = \frac{1}{N} \sum_{i=1}^N u_{ii} \leq \frac{1}{N} \sum_{i=1}^N u'_{ii} = {\mathrm{trace}}(\mathbf{U}').
    $

\end{enumerate}
Thus, ${\mathrm{trace}}$ satisfies all the conditions set in Definition \autoref{def:mao}. Indeed, given a subset of matrix positions $S \subset \{1,\ldots,N\} \times \{1,\ldots,K\}$ and a flattening operator $f:S\longrightarrow\{1,\ldots,|S|\}$, it is easy to see that the matrix operator defined for any $\mathbf{U} \in \mathcal{M}_{N,K}([0,1])$ as $F_{f,S,A}(\mathbf{U})=A(u_{f^{-1}(1)},\ldots,u_{f^{-1}(|S|)})$ is a MAO if $A$ is an AO of dimension $|S|$. The trace mapping is a particular case of $F_{f,S,A}$ operator for which $K=N$, $S$ is the set of diagonal entries, $f$ is any flattening function, and $A$ is the arithmetic mean. However, not all usual matrix operators are MAOs. As an instance, note that the determinant of matrices in $\mathcal{M}_{N,N}([0,1])$ clearly fails to verify conditions 2 and 3 of Definition \autoref{def:mao}, and can also be negative (i.e., it does not range into [0,1]).
\end{example}

\begin{remark}\label{rem_monoton}
Note that Definition~\ref{def:mao} only imposes non-strict monotonicity with respect to the natural entrywise order, and does not require $F$ to be sensitive to every entry of the matrix. Example~\ref{ex:trace} illustrates this fact: the $\mathrm{trace}$ mapping (resp. any $F_{S,A}$ operator) is monotone according to Definition~\ref{def:mao}; however it disregards all off-diagonal entries (resp. all entries not in $S$), only needing to be non-decreasing along the diagonal (resp. along $S$). This suggests that it could make sense to study more general classes of matrix operators for which monotonicity is restricted to a subset $S$ of the matrix entries. Moreover, the class of MAOs can be also relaxed through the notion of directional monotonicity, introduced for vector-based aggregation functions in \cite{bustince2015directional}. This would relate MAOs with quasi- and pseudo-aggregation functions. These relaxations are not further addressed here and are left as a subject for future research.
\end{remark}

\subsection{Decomposability of Matrix Aggregation Operators}\label{sec:decomposable}

\noindent Although Definition \ref{def:mao} presents aggregation operators acting on matrices, constructing MAOs directly may be difficult in practice. A natural and intuitive approach for their construction consists of exploiting the inherent two-dimensional structure of the data by performing the aggregation in two successive stages. Specifically, one may first aggregate the entries of each row (or each column) into a single value and subsequently aggregate the resulting vector. This decomposition preserves the semantic distinction between the two matrix axes while allowing the use of well-established vectorial aggregation operators in each stage. Such a decomposability constitutes an straightforward and widely applicable way of building MAOs. Nevertheless, as will be shown later, decomposable MAOs represent only a particular subclass of all MAOs.

\begin{definition}[Row-decomposable MAO]\label{def:row_decomposable}
 A MAO  $F : \mathcal{M}_{N,K}([0,1]) \longrightarrow [0,1]$ is said to be \textit{row-decomposable} when there exist AOs $C:[0,1]^N\longrightarrow [0,1]$ and $R_1,\dots,R_N:[0,1]^K\longrightarrow [0,1]$ such that 
 $$
F(\textbf{U}) = C(R_1(\textbf{r}_1),\dots,R_N(\textbf{r}_N))
 $$
 for any $\textbf{U}\in\mathcal{M}_{N,K}([0,1])$, where $\textbf{r}_i$ denotes the $i$-th row of $\textbf{U}$.
\end{definition}

\begin{definition}[Column-decomposable MAO]
 A MAO  $F : \mathcal{M}_{N,K}([0,1]) \longrightarrow [0,1]$ is said to be \textit{column-decomposable} when there exist AOs $R:[0,1]^K\longrightarrow [0,1]$ and $C_1,\dots,C_K:[0,1]^N\longrightarrow [0,1]$ such that 
 $$
F(\textbf{U}) = R(C_1(\textbf{c}_1),\dots,C_K(\textbf{c}_K))
 $$
 for any $\textbf{U}\in\mathcal{M}_{N,K}([0,1])$, where $\textbf{c}_j$ denotes the $j$-th column of $\textbf{U}$.
\end{definition}

\begin{definition}[Decomposable MAO]
A MAO is said to be decomposable if it is either row-decomposable or column-decomposable.
\end{definition}

Note that the trace operator introduced in Example \ref{ex:trace} is a decomposable MAO, as it is indeed both row-decomposable and column-decomposable. This follows by taking $R_i$ (resp. $C_i$) as the projection into the $i$-th axis, and $C$ (resp. $R$) as the arithmetic mean.
 
It is important to stress that, in the above terms, the notion of decomposable MAO is different from that of MAO itself, since there exist some MAOs that are not decomposable. 
 
\begin{example} [Non-decomposable MAO]\label{ex_nondecom}
 The function $F_1:\mathcal{M}_{2,2}([0,1]) \longrightarrow [0,1]$ given by
\begin{equation}\label{eq_ndmao}
  F_1(\textbf{U}) = \min(u_{11}\cdot u_{22} + u_{21} \cdot u_{12},1)
\end{equation}
\noindent for any $\textbf{U}\in\mathcal{M}_{2,2}([0,1]) $ is clearly a MAO, but it is neither row-decomposable nor column-decomposable, as shown in \ref{appendix}. It has to be stressed that the operator $F_1$ can be regarded as a particular instance of the general polarization index $JDJ$ proposed in \cite{guevara20} (see also Example \ref{ex:polarization}).

Similarly, the notions of row-decomposability and column-decomposability are also different. Assuming any $\textbf{U}\in\mathcal{M}_{2,2}([0,1])$, an instance of  row-decomposable MAO that is not column-decomposable is
\begin{equation*}
F_2(\textbf{U}) = \min(u_{11}\cdot u_{12} + u_{21} \cdot u_{22},1) = C(R_1(u_{11},u_{12}),R_2(u_{21},u_{22})) 
\end{equation*}
\noindent with $R_1(x,y)=R_2(x,y)=xy$ and $C(x,y)=\min(x+y,1)$. While a column-decomposable MAO that is not row-decomposable is given by
\begin{equation*}
F_3(\textbf{U}) = \min(u_{11}\cdot u_{21} + u_{12} \cdot u_{22},1) = R(C_1(u_{11},u_{21}),C_2(u_{12},u_{22}))
\end{equation*}
with $C_1(x,y)=C_2(x,y)=xy$ and $R(x,y)=\min(x+y,1)$. 
\end{example}

The following result shows that decomposable MAOs can be constructed by simply specifying the AOs that appear in the corresponding decomposition. Attending to Remark \ref{rem_maoiffao}, its proof is straightforward taking into account that the composition of AOs is also an AO.

\begin{proposition}\label{prop1}
If $C:[0,1]^N\to [0,1]$ and $R_1,\dots,R_N:[0,1]^K\to [0,1]$ are aggregation operators, then the function $F : \mathcal{M}_{N,K}([0,1]) \to [0,1]$ given by
$$
F(\textbf{U}) = C(R_1(\textbf{r}_1),\dots,R_N(\textbf{r}_N))
$$
for any $\textbf{U}\in\mathcal{M}_{N,K}([0,1])$ is a row-decomposable MAO.

Similarly, if $R:[0,1]^K\to [0,1]$ and $C_1,\dots,C_K:[0,1]^N\to [0,1]$ are aggregation operators, then the function $F : \mathcal{M}_{N,K}([0,1]) \to [0,1]$ given by
$$
F(\textbf{U}) = R(C_1(\textbf{r}_1),\dots,C_K(\textbf{r}_K))
$$
for any $\textbf{U}\in\mathcal{M}_{N,K}([0,1])$ is a column-decomposable MAO.
\end{proposition}

\section{Symmetry of Matrix Aggregation Operators}\label{sec:symmetry}
\noindent In this section, we define several symmetry properties of MAOs and study the relationships between them. In this context, by symmetry it is mean the invariance of a MAO under permutations of some groups of the input matrix elements, particularly permutations of rows, columns, or combinations thereof
\footnote{Although other classes of permutations of matrix elements do exist, such as rotations or transpositions, the discussion is restricted to row and column permutations mainly due to the fact that, as discussed in the Introduction section, in the context of MAOs matrix rows and columns are tipically associated to practically relevant elements such as objects, classes, rules, decision-makers, criteria, etc. (see Section \ref{sec:examples} for some more instances). Thus, it is natural to start the study of symmetry of MAOs by focusing on operators which are invariant under permutations of these elements.}. 

\subsection{Global Symmetry}
\noindent Some notation and conventions are first introduced. Let $\mathcal{S}_L$ denote the set of all permutations of $L$ elements. Given a permutation $\sigma \in \mathcal{S}_L$ and a vector $\textbf{v}\in [0,1]^L$, we denote by $\sigma(\textbf{v})$ the length-$L$ vector obtained by permuting the indexes of $\textbf{v}$ according to $\sigma$, i.e. $\sigma(\textbf{v})_l=v_{\sigma(l)}$ for all $l=1,\dots,L$. We assume  $\sigma(\textbf{v})$ has the same shape as $\textbf{v}$: if $\textbf{v}$ is arranged as a row then so is $\sigma(\textbf{v})$, and conversely if $\textbf{v}$ is arranged as a column. Similarly, given an $N\times K$ matrix $\textbf{U}=(u_{ij})\in \mathcal{M}_{N,K}([0,1])$ and a permutation $\sigma\in\mathcal{S}_{N\cdot K}$, we denote by $\sigma(\textbf{U}) \in \mathcal{M}_{N,K}([0,1])$ the matrix obtained by permuting the elements $u_{ij}$ of $\textbf{U}$ according to $\sigma$, sorting them lexicographically using subindexes $i$ and $j$.
    
In these terms, the standard symmetry of a (vectorial) AO $A:[0,1]^{n}\to [0,1]$ can be stated as 
\begin{equation}\label{eq:vecorialAO}
  A(\textbf{v})=A(\sigma(\textbf{v}))
\end{equation}
\noindent for any vector $\textbf{v}\in [0,1]^{n}$ and any permutation $\sigma\in\mathcal{S}_{n}$. A first notion of symmetry for MAOs is obtained as the direct translation of this property.

\begin{definition}[Global symmetry (GS)]\label{def:globalSymmetry}
  Let $F:\mathcal{M}_{N,K}([0,1])\to [0,1]$ be a MAO. $F$ is said to be globally symmetric  if for any $\textbf{U}\in\mathcal{M}_{N,K}([0,1])$ and any permutation $\sigma\in\mathcal{S}_{N\cdot K}$ it holds that
  $$
    F(\textbf{U})=F(\sigma(\textbf{U})). 
  $$
\end{definition}

\begin{example}[Symmetric matrix aggregation operators]
  The following are examples of MAOs 
 verifying Definition \ref{def:globalSymmetry}:
  \begin{itemize}
    \item $F_4(\textbf{U}) = \max_{i,j}u_{ij}$.
    \item $F_5(\textbf{U}) = \frac{1}{N\cdot K}\sum_{i=1}^N\sum_{j=1}^{K}u_{ij}$.
    \item $F_6(\textbf{U}) = \prod_{i,j}u_{ij}$.
  \end{itemize}
\end{example}

\begin{remark}\label{rem_symflat}
  Note that when $F$ is a globally symmetric MAO its flattened versions $F_f$ (introduced in Remark \ref{rem_maoiffao}) do not depend on $f$, meaning that any flattening operator $f$ always produce the same AO. Moreover, in such a case this unique flattened AO has to be symmetric. It is easy to see that the converse is also true: if $F_f$ is symmetric for a given $f$, then $F$ has to verify GS. Therefore, GS is equivalent to the invariance of the flattened AOs $F_f$, and in turn also to the symmetry of this unique $F_f$.
\end{remark}

In view of Remark \ref{rem_symflat}, it should be stressed that when GS holds, the matrix structure of the input data $\textbf{U}\in\mathcal{M}_{N,K}([0,1])$ becomes irrelevant. This is because in such a case, there is no difference between applying a globally symmetric MAO $F$ to $\textbf{U}$ and applying the symmetric vectorial aggregation $F_f$ to an arbitrarily flattened version of $\textbf{U}$. Furthermore, the global symmetry of $F$ as a MAO has the same meaning as the symmetry of $F_f$—any subset of elements can be permuted without affecting the result of the aggregation.

However, other kinds of symmetry or invariance can be devised that are only meaningful in a matrix setting. These are related to permutations of specific subsets of elements that can only be defined within a matrix context, such as rows and columns, and are thus not applicable or meaningless when applied to flattened vectors.

\subsection{Row and Column Symmetry}\label{sec:RSprops}
\noindent Let us recall that any matrix $\mathbf{U}$ can be expressed in terms of either its rows $\mathbf{r}_i \in [0,1]^K$ or columns $\mathbf{c}_j \in [0,1]^N$, $i=1,\dots,N$, $j=1,\ldots,K$, as a sum of their outer products with the appropriate canonical vectors. That is, if $\mathbf{e}_{L,l}$ denotes the length-$L$ canonical column vector having all its elements equal to 0 except the $l$-th one, which equals 1, then it holds that

\begin{equation}\label{eq:matrowcol}
  \mathbf{U}=\sum_{i=1}^N \mathbf{e}_{N,i}\cdot \mathbf{r}_i = \sum_{j=1}^K \mathbf{c}_j\cdot \mathbf{e}_{K,j}^t,
\end{equation}

\noindent 
where in this case the symbol $\cdot$ acts as the outer product between a $N\times 1$ vector and a $1\times K$ vector, thus producing a $N\times K$ matrix in each addend.

To analyze symmetry properties associated with the invariance of a MAO under permutations of the rows (resp.  columns) of the input matrix $\mathbf{U}$, two kinds of permutations need to be considered:

\begin{itemize}
  \item \textbf{External}: Rows (resp. columns) are permuted between them, exchanging their position in $\textbf{U}$. In the terms of Eq.\ref{eq:matrowcol}, this is accomplished by applying a permutation $\sigma\in\mathcal{S}_N$ (resp. $\theta\in\mathcal{S}_K$) to the indexes of the canonical vectors, thus substituting $\textbf{e}_{N,i}$ by $\textbf{e}_{N,\sigma(i)}$ (resp. $\textbf{e}^t_{K,j}$ by $\textbf{e}^t_{K,\theta(j)}$) in Eq.\ref{eq:matrowcol}.

  \item \textbf{Internal}: Elements $u_{ij}$ of each row $\textbf{r}_i$ (resp. column $\textbf{c}_j$) are permuted, with $j=1,\dots,K$ (resp. $i=1,\ldots,N$). To this aim, a permutation $\theta_i\in\mathcal{S}_K$ (resp. $\sigma_j\in\mathcal{S}_N$) is applied to each row (column), thus substituting $\textbf{r}_i$ by $\theta_i(\textbf{r}_i)$ (resp. $\textbf{c}_j$ by $\sigma_j(\textbf{c}_j)$) in Eq.\ref{eq:matrowcol}.
\end{itemize}

\noindent Note that while a different internal permutation $\theta_i$ (resp. $\sigma_j$) can be applied to each row (resp. column), it only makes sense to consider a single external permutation $\sigma$ (resp. $\theta$) acting on all the rows (resp. columns). 

From these notions, four basic row and column symmetry properties can be introduced. 

\begin{definition}\label{def:symprop}
Let $F:\mathcal{M}_{N,K}([0,1])\to [0,1]$ be a MAO and let $\textbf{U}\in\mathcal{M}_{N,K}([0,1])$ be any input matrix. Then:
\begin{itemize}
  \item It is said that $F$ verifies \textit{internal row symmetry} (\textbf{IRS}) if for any permutations $\theta_i\in\mathcal{S}_K$, $i=1,\dots,N$, it holds that
  \begin{equation}
    F(\textbf{U})=F(\sum_{i=1}^N\textbf{e}_{N,i}\cdot\theta_i(\textbf{r}_i)).
  \end{equation}
  \item It is said that $F$ verifies \textit{external row symmetry} (\textbf{ERS}) if for any permutation $\sigma\in\mathcal{S}_N$ it holds that
  \begin{equation}
    F(\textbf{U})=F(\sum_{i=1}^N\textbf{e}_{N,\sigma(i)}\cdot\textbf{r}_i).
  \end{equation}
  \item It is said that $F$ verifies \textit{internal column symmetry} (\textbf{ICS}) if for any permutations $\sigma_j\in\mathcal{S}_N$, $j=1,\dots,K$, it holds that
  \begin{equation}
    F(\textbf{U})=F(\sum_{j=1}^K \sigma_j(\textbf{c}_j) \cdot \textbf{e}^t_{K,j}).
  \end{equation}
  \item It is said that $F$ verifies \textit{external column symmetry} (\textbf{ECS}) if for any permutation $\theta\in\mathcal{S}_K$ it holds that
  \begin{equation}
    F(\textbf{U})=F(\sum_{j=1}^K \textbf{c}_j \cdot \textbf{e}^t_{K,\theta(j)}).
  \end{equation}
  \end{itemize}
\end{definition}
Note that applying the same internal permutation $\theta_i=\theta$ for $i=1,\ldots,N$ (resp. $\sigma_j=\sigma$ for $j=1,\ldots,K$) to all rows (resp. columns) is equivalent to its application as an external permutation of columns (resp. rows), as expresed by the identity
\begin{equation}
\sum_{i=1}^{N}\textbf{e}_{N,\sigma(i)}\cdot\theta(\textbf{r}_i) =
\sum_{j=1}^{K}\sigma(\textbf{c}_j)\cdot\textbf{e}^t_{K,\theta(j)}.
\end{equation}
Therefore, ICS is a special case of ERS, and similarly IRS is a particular instance of ECS, as stated in the following result.
\begin{proposition}\label{prop_IErel}
    Let $F$ be a MAO. If $F$ verifies ICS, then it also fulfills ERS. And if $F$ verifies IRS, then it also fulfills ECS.
\end{proposition}

Besides these relationships, the four basic row and column symmetry properties are not mutually exclusive, meaning that a MAO may fulfill more than one simultaneously. This motivates the following definitions.

\begin{definition}\label{def:symprop}
Let $F:\mathcal{M}_{N,K}([0,1])\to [0,1]$ be a MAO and let $\textbf{U}\in\mathcal{M}_{N,K}([0,1])$ be any input matrix. Then:
\begin{itemize}
  \item It is said that $F$ verifies \textit{total row symmetry} (\textbf{TRS}) if it simultaneously satisfies IRS and ERS, that is, if for any permutations $\sigma\in\mathcal{S}_N$ and $\theta_i\in\mathcal{S}_K$, $i=1,\dots,N$, it holds that
  \begin{equation}
    F(\textbf{U})=F(\sum_{i=1}^N\textbf{e}_{N,\sigma(i)}\cdot\theta_i(\textbf{r}_i)).
  \end{equation}
  \item It is said that $F$ verifies \textit{total column symmetry} (\textbf{TCS}) if it simultaneously satisfies ICS and ECS, that is, if for any permutations $\theta\in\mathcal{S}_K$ and $\sigma_j\in\mathcal{S}_N$, $j=1,\dots,K$, it holds that
  \begin{equation}
    F(\textbf{U})=F(\sum_{j=1}^K \sigma_j(\textbf{c}_j) \cdot \textbf{e}^t_{K,\theta(j)}).
  \end{equation}
  \item It is said that $F$ verifies \textit{external row and column symmetry} (\textbf{ERCS}) if it simultaneously satisfies ERS and ECS, that is, if for any permutations $\theta\in\mathcal{S}_K$ and $\sigma\in\mathcal{S}_N$ it holds that
  \begin{equation}
    F(\textbf{U})=F(\sum_{i=1}^N\textbf{e}_{N,\sigma(i)}\cdot\theta(\textbf{r}_i))=F(\sum_{j=1}^{K}\sigma(\textbf{c}_j)\cdot\textbf{e}^t_{K,\theta(j)}).
  \end{equation}
  \end{itemize}
\end{definition}

This last definition addresses the combinations of basic properties given by IRS $\land$ ERS, ICS $\land$ ECS, and ERS $\land$ ECS. To see what happens with the combination IRS $\land$ ICS, let us recall that GS is associated with general, arbitrary permutations $\sigma\in\mathcal{S}_{N\cdot K}$ of the elements of a matrix $\textbf{U}\in\mathcal{M}_{N,K}([0,1])$. It is important to note that any such permutation $\sigma$ can be expressed as the composition of internal row and column permutations $\theta_i,\ \sigma_j$ when these are allowed to be different for each row $\textbf{r}_i$ and column $\textbf{c}_j$, $i=1,\dots,N,\ j=1,\dots,K$\footnote{Technically, the algebraic group generated by independent permutations of entries in each row and each column equals the full symmetric group on the whole $N\cdot K$ matrix entries. The main idea behind this result is that any pair of arbitrary matrix elements can be transposed using specific internal row and column permutations of the entries in the $2\times2$ block associated with such pair of elements. Consequently, since the set of all transpositions generates the full symmetric group (see e.g. \cite{DummitFoote}), it follows that $\mathcal{S}_{N\cdot K}$ is generated by the set of all internal row and column permutations $\theta_i,\ \sigma_j$, $i=1,\dots,N,\ j=1,\dots,K$.}. The following result is then straightforward.

\begin{proposition}\label{prop:equivGS}
Let $F:\mathcal{M}_{N,K}([0,1])\to [0,1]$ be a MAO. Then, $F$ verifies GS if and only if it verifies both IRS and ICS.
\end{proposition}

This result can be understood as imposing a limit on the simultaneous verification of row and column symmetry conditions, after which the matrix structure no longer matters in the aggregation process. As exposed above, this is because a MAO for which both IRS and ICS simultaneously hold is globally symmetric. Such a MAO can actually be regarded as a symmetric vectorial aggregation operator on $[0,1]^{N\cdot K}$ rather than as a MAO, since the entries of the matrix to be aggregated can then be rearranged arbitrarily. Any connection to the underlying meaning of rows and columns is thus lost in this case.

Proposition \ref{prop_IErel} entails that the remaining pairwise combinations IRS $\land$ ECS and ICS $\land$ ERS reduce to just IRS and ICS, respectively. This also entails that combinations of more than two basic symmetry properties also reduce to the considered cases. For instance, ICS $\land$ ECS $\land$ ERS reduces to ICS $\land$ ECS = TCS. Similarly, ICS $\land$ IRS $\land$ ERS reduces to ICS $\land$ IRS = GS. 

Finally, a similar argument to that leading to Proposition \ref{prop_IErel} allows establishing both TRS and TCS as special cases of ERCS.

\begin{proposition}\label{prop_TRCS}
    Let $F$ be a MAO. If $F$ verifies either TRS or TCS, then it also verifies ERCS.
\end{proposition}
\begin{proof}
    If $F$ verifies TRS, then by definition both IRS and ERS have to hold. But by Prop. \ref{prop_IErel}, IRS entails ECS, and thus $F$ verifies both ERS and ECS, which is the defining condition of ERCS. A similar argument applies when $F$ verifies TCS.
\end{proof}
    
In conclusion, the combination of the four basic row and column symmetry conditions allows generating four more additional symmetry properties, for a total of eight properties, including GS. According to Definition \ref{def:symprop} and Propositions \ref{prop_IErel}-\ref{prop_TRCS}, relationships between these 8 properties can be summarized as shown in \autoref{fig:relprops}.

\begin{figure}[ht]
  \centering
  \includegraphics[width=0.5\textwidth]{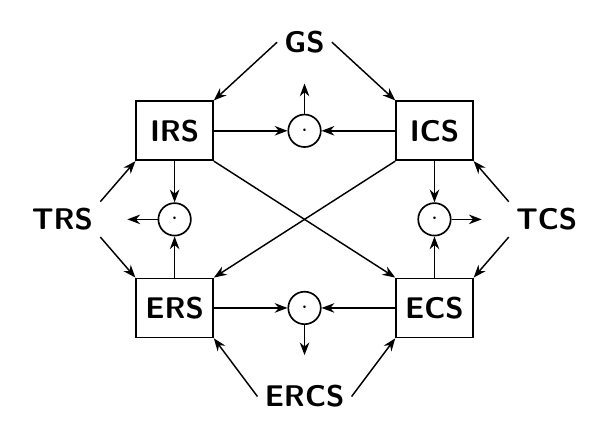}
  \caption{Relationships between the four basic row and column symmetry properties (within boxes) and the derived symmetry properties, including global symmetry. Arrows should be read as implications, while the symbol $\odot$ denotes the conjunction of properties.}
  \label{fig:relprops}
\end{figure}

Finally, the next result states the relationships that hold between row and column symmetry properties when the input matrix is transposed.

\begin{proposition}
  Let $F:\mathcal{M}_{N,K}([0,1])\to [0,1]$ be a MAO of dimension $N\times K$, and let $G:\mathcal{M}_{K,N}([0,1])\to [0,1]$ be the function defined as $G(\textbf{U})=F(\textbf{U}^t)$ for any $\textbf{U} \in \mathcal{M}_{K,N}([0,1])$. Then, $G$ is a MAO of dimension $K\times N$, and the following holds:
  \begin{itemize}
    \item $G$ verifies ERS (resp. ECS) if and only if $F$ verifies ECS (resp. ERS).
    \item $G$ verifies IRS (resp. ICS) if and only if $F$ verifies ICS (resp. IRS).
  \end{itemize}
\end{proposition}
\begin{pf}
Given $\textbf{U} \in \mathcal{M}_{K,N}([0,1])$, the result follows directly from the fact that, if $\textbf{r}_i$ and $\textbf{c}_j$ respectively denote the rows and columns of $\textbf{U}^t$, then
$$
\textbf{U}=\sum_{j=1}^{K}\textbf{e}_{K,j}\cdot\textbf{c}_j = \sum_{i=1}^{N}\textbf{r}_i\cdot\textbf{e}^t_{N,i}.
$$
\end{pf}


\begin{example}[\bf Symmetry properties]\label{ex:SymConds} The trace operator described in Example \ref{ex:trace} constitutes an instance of a MAO that does not verify any of the eight exposed symmetry properties. Another example is given by 
$$
F_7(\textbf{U})=\sum_{i=1}^N w_i(\sum_{j=1}^K v_j u_{ij})=\mathbf{w}^t\mathbf{U}\mathbf{v}
$$
\noindent for general weighting vectors $\textbf{v}\in[0,1]^K$ and $\textbf{w}\in[0,1]^N$.

Next we provide instances of MAOs that do not verify GS but do fulfill some of the other basic or derived symmetry properties. In this way, it is easy to see that for any $p>0$ TRS holds for
$$
F_9(\textbf{U})=\prod_{i=1}^N(\max\textbf{r}_i)^p,
$$
\noindent and also that it verifies ERCS but not TCS. Similarly, 
$$
F_{10}(\textbf{U})=\prod_{j=1}^K(\max\textbf{c}_j)^p
$$
\noindent verifies TCS and ERCS, but not TRS. A MAO that fullfils ERS  but not ERCS is
$$
F_{11}(\textbf{U})=\prod_{i=1}^{N}\sum_{j=1}^Kv_ju_{ij},
$$
\noindent while
$$ F_{12}(\textbf{U})=\sum_{i=1}^Nw_i\min\textbf{r}_i
$$
\noindent verifies IRS but not ERS. Finally, ECS holds for
$$
F_{13}(\textbf{U})=\frac{1}{K}\sum_{j=1}^K\sum_{i=1}^K w_i u_{ij}, 
$$
\noindent but it does not verify IRS.

All the instances provided so far in this example are decomposable MAOs. Although one may conjecture that it is hard that any of the row or column symmetry conditions studied in this section hold for general non-decomposable MAOs, note that the non-decomposable operator $F_1$ introduced in Eq.\eqref{eq_ndmao} indeed verifies ERCS. However, $F_1$ can be easily turned into an operator that does not fulfill any symmetry property by introducing weights in its definition, and thus the inverse conjecture also does not hold in general.
\end{example}

\begin{remark}
  Note that while the previous example shows that the trace operator does not verify any of the considered symmetry properties, it is nevertheless invariant under any permutation of the matrix diagonal entries. We refer to this invariance under permutations of the diagonal entries as \emph{diagonal symmetry} (DS). 
  Therefore, there are alternative symmetry properties for MAOs, such as DS, which are not covered by the four basic row and column symmetry properties. Indeed, following the idea of Remark \ref{rem_monoton}, it is possible to devise symmetry properties associated with specific subsets of matrix entries, such that a MAO just has to be invariant under permutations of these entries. No one of such subset-based properties (except of course GS) is a particular case of the exposed row and column symmetry properties, while some of the former ones generalize some of the latter (e.g., internal symmetry of just a single row generalizes IRS). However, in our oppinion, the practical interest of these alternative symmetry properties does not match that of the row and column ones, and they are not further addressed here.
\end{remark}

\section{Matrix Aggregation Operators in applications}\label{sec:examples}

\noindent In this section, examples are provided of well-known MAOs naturally arising in real problems and applications where information is represented by a matrix of membership degrees, scores, or evaluations. The purpose of this section is not to provide an exhaustive review of these areas, but rather to illustrate that many well-established aggregation procedures can be naturally interpreted as MAOs.

\begin{example}[\bf Multicriteria Decision Making Problems]

\noindent Let $X = \{a_1, \dots, a_N \}$ denote a set of alternatives, and let $C_1, \dots, C_K$ denote a set of criteria. Let us denote by $u_{ij}$ the degree to which the alternative $a_i$ reaches the objective (or satisfies the criterion) $C_j$.  
There are many multicriteria decision making (MCDM) methods in which the information is represented as a matrix $\mathbf{U}\in\mathcal{M}_{N,K}([0,1])$, where $u_{ij} = 1$ means that the alternative $a_i$ completely satisfies the criterion $C_j$. In this framework, decision procedures often involve the aggregation of all matrix entries into a single representative value per alternative, leading to the selection of the most appropriate alternative according to some aggregation rule. Formally, the aggregated utility $u^*$ of the chosen alternative can be associated to the result of applying a matrix aggregation operator $F$ on the matrix \textbf{U}, such that $u^*=F(\textbf{U})$. Some aggregation functions that can be found in literature and that corresponds with different decision strategies are:

\begin{enumerate}
\item \textbf{Max--Min strategy} $F(\mathbf{U})=\max_i\left\{\min_j u_{ij}\right\}$. 
This operator satisfies TRS but not GS.

\item \textbf{Min--Max strategy}  Assuming that $u_{ij}$ represents a risk or cost, $
F(\mathbf{U})=\min_i\left\{\max_j u_{ij}\right\}$. This operator TRS but not GS.

\item \textbf{Weighted Sum strategy} $F(\mathbf{U})=\max_i\left\{\sum_{j=1}^{K}w_j u_{ij}\right\}. $ When the weights remain associated with fixed criteria, this operator only satisfies ERS.

\item \textbf{MaxiMax strategy} $ F(\mathbf{U})=\max_i\left\{\max_j u_{ij}\right\}=\max u_{ij}$. This operator is globally symmetric. 

\item \textbf{Ordered Weighted Averaging (OWA) strategy} $ F(\mathbf{U})=$ \\ $\max_i\left\{\sum_{j=1}^{K}w_j u_{i(j)}\right\}$.  This operator always satisfies TRS, and is globally symmetric when the row-level OWA operator (within brackets) reduces to the maximum operator for each row.
\end{enumerate}
\end{example}

\begin{example}[\bf Decision making with multiple DMs.]

\noindent Let $X = \{a_1, \dots, a_n\}$ denote a set of alternatives, and let $\{D_1, \dots, D_k\}$ denote a set of decision makers (DMs). If we define $u_{ij}$ as the degree to which the alternative $a_i$ satisfies the main objective (or preference) of decision maker $D_j$, we have that, from a mathematical point of view, this is formally equivalent to the classical Multi-Criteria Decision-Making (MCDM) framework. In the MCDM case, one defines a set of criteria $C_1, \dots, C_k$ instead of decision makers, and the same notation applies. Hence, the mathematical formulation of the multi-DM problem coincides with that of MCDM, differing only in the interpretation of the index $j$. In both cases, decision procedures typically involve aggregating the evaluations of all decision makers (or criteria) into a single representative value for each alternative.  
\end{example}

\begin{example}[\bf Fuzzy Rule-Based Classification Systems]

\noindent In the context of fuzzy rule-based classification systems (FRBCS), the knowledge base can be assumed to be composed of a set of $Q$ fuzzy rules of the form
$$
R_q: \text{ If } X_1 \text{ is } A^q_1 \text{ and } X_2 \text{ is } A^q_2 \text{ and} \dots \text{and } X_m \text{ is } A^q_m, \text{ then Class} = (r_1^q,\dots,r_K^q),
$$
where $K$ is the number of classes, and for each $q=1,\dots,Q$, $A^q_i$ denotes the fuzzy set associated with the feature $X_i$ in the $q$-th rule, and $r_j^q\in[0,1]$ represents the weight or evidence level associated with the $j$-th class by that rule.

Given a query or input pattern $\mathbf{x} = (x_1, \dots, x_m)$ to be classified, the standard fuzzy reasoning method \cite{Cordon1999} first multiplies the degree of compatibility $\mu_q(\mathbf{x})$ 
by the rule weights of all rules, obtaining the association degrees $ u_{qj} = \mu_q(\mathbf{x}) \cdot r_j^q,\ q=1,\dots,Q,\ j=1,\dots,K,$. These quantify the evidence for each class according to each rule. The association degrees of each class $C_j$ are then aggregated into a certainty degree $ \pi_j(\textbf{x})=A(u_{1j},\dots,u_{Qj}) \in [0,1]$, and the query \textbf{x} is finally asigned to the class with maximum certainty degree. Typical choices for the aggregation operator $A$ are the maximum (leading to a classical, winner-rule FRBCS) or the mean. Note then that the certainty degree associated to the predicted class is given by the MAO $F(\textbf{U})=\max(A(\mathbf{c}_1),\ldots,A(\mathbf{c}_K))$, where $\mathbf{c}_j$ denotes the $j$-th column of the $Q \times K$ matrix of association degrees $\textbf{U}=(u_{qj})$. For the mentioned choices of $A$, $F$ is globally symmetric when $A=\max$, while for $A=$ mean it verifies TCS.
\end{example}

\begin{example}[\bf Grouping and Overlap Indices]

\noindent Based on the notion of consistency index \cite{Zadeh1971} (a measure of similarity between fuzzy sets), some authors \cite{garcia2014overlap, dimuro2024grouping} introduced, for two sets, the concepts of Overlap Index ($OI$), which quantifies the degree of overlap between fuzzy sets in a given universe, and its dual, the Grouping Index ($GI$). In \cite{asmus2025generalized} the concepts of Grouping and Overlap Index were considered in the $K$-dimensional case, where $K$ fuzzy sets $A_1, \dots, A_K$ are defined over the same universe $X=\{x_1, \dots, x_N\}$ through membership functions $\mu_{A_j}:X \to [0,1]$. Then, from a mathematical point of view, both $OIs$ and $GIs$ can be viewed as MAOs that transform the matrix $\mathbf{U}=(u_{ij})$ with elements $u_{ij} = \mu_{A_j}(x_i)$ into a value in $[0,1]$.
\end{example}

\begin{example}[\bf Fuzzy Clustering]\label{ex:fuzzyClustering}

\noindent In a fuzzy clustering context, a set of objects $X = \{x_1,  \dots, x_N\}$ has to be partitioned into $K$ fuzzy clusters $\{C_1,\dots,C_K\}$. For each object $x_i$ and each cluster $C_j$, let $u_{ij} \in [0,1]$ express the extent up to which $x_i$ belongs to or is associated with cluster $C_j$. All the fuzzy clustering information is thus sumarized by means of the matrix $\mathbf{U} = (u_{ij})$. To evaluate the quality of a fuzzy partition, many authors develop indices in which all the information contained in $\mathbf{U}$ is aggregated into a single value that reflects the overall quality of the partition, for instance, through proxy notions of compactness and separability of the clusters. Several fuzzy clustering validity indices, such as, e.g., the Xie–Beni (XB) index \cite{xieBeni1991validity}, can thus be reagarded as instances of MAOs. 
\end{example}

\begin{example}[\bf Polarization Measures]\label{ex:polarization}

Given a set of $N$ individuals and two poles $X_A$ and $X_B$, let $\mathbf{U}\in\mathcal{M}_{N,2}([0,1])$ be the matrix such that $u_{i1}=\mu_{X_A}(x_i)$ and $u_{i2}=\mu_{X_B}(x_i)$ represent the membership degree of individual $x_i$ to each of the poles, for any $i=1,\ldots,N$. The JDJ polarization index defined in \cite{guevara20,gutierrez2021community} can be directly expressed as a function of $\mathbf{U}$, given by
$$
JDJ_{pol}(\mathbf{U}) = \frac{2}{N(N-1)}\sum_{i=1}^{N-1}\sum_{j>i} \phi\big(\varphi(u_{i1},u_{j2}),\, \varphi(u_{j1},u_{i2})\big),
$$
where $\varphi$ and $\phi$ are, respectively, an overlap and a grouping function. Unlike the previous examples, $JDJ_{pol}$ genuinely combines information from different rows and columns of $\mathbf{U}$ through the bivariate functions $\varphi$ and $\phi$, and similarly to the the MAO introduced in Example \ref{ex_nondecom} it can not be decomposed into a pair of successive row and column aggregations. Thus, it constitutes a general example of non-decomposable MAO.
\end{example}

Although the application domains of the examples above differ substantially, they all share the same mathematical structure: the available information is naturally represented by a matrix whose rows and columns have different semantic meanings. Matrix Aggregation Operators provide a unified framework for studying these aggregation processes while preserving the inherent two-dimensional structure of the data.

\section{Maximum Entropy Global Coverage Indexes}\label{sec:covering}
\noindent This section builds upon some preliminary notion, exposed in \cite{eusflat25:covering,rodriguez2025crb}, in order to introduce an interesting particular class of MAOs, to be referred to as maximum entropy global coverage indexes (MEGCIs). Specifically, here we depart from the notion of global coverage index (GCI), which was proposed in \cite{eusflat25:covering} in a context where elements $u_{ij}\in[0,1]$ of the matrix $\mathbf{U}$ were assumed to measure the level up to which an object $x_i$ is represented (or covered) by the $j$-th of a set of $K$ possible classes or representatives. Next, this notion is more abstractly and concisely reformulated in the framework of MAOs.

\begin{definition}[Global Coverage Index \cite{eusflat25:covering}]
A MAO $F : \mathcal{M}_{N,K}([0,1]) \to [0,1]$ is called a \textit{Global Coverage Index} if it is totally row symmetric (i.e., if it verifies TRS), and satisfies the following additional properties:
\begin{enumerate}
\item[$(GCI1)$] $F(\mathbf{U}) = 0$ if and only if $\max \mathbf{r}_{i} = 0$ for all $i = 1, \dots, N$.
\item[$(GCI2)$] $F(\mathbf{U}) = 1$ if and only if $\max\mathbf{r}_{i} = 1$ for all $i = 1, \dots, N$.
\end{enumerate}
\end{definition}

Thus, a GCI takes value 1 iff each row of $\mathbf{U}$ contains an element equal to 1, and it takes value 0 iff all elements of $\mathbf{U}$ are 0. Moreover, a GCI has to be invariant under external permutations of objects, as well as under internal permutation of classes for each object. 

In \cite{eusflat25:covering}, a method providing sufficient conditions for constructing row-decom- posable GCIs was provided. Given a function $G:[0,1]^K\to[0,1]$, let $\mathbf{G}:\mathcal{M}_{N,K}([0,1])\to [0,1]^N$ be an $N$-dimensional vectorial map defined as $\mathbf{G}(\mathbf{U})=(G(\mathbf{r}_1),\ldots,G(\mathbf{r}_N))$. Then, the construction method is based on compositions of the form $(M\circ\mathbf{G})$, with $G$ being a $K$-dimensional grouping function and $M$ an aggregation operator with a particular averaging and compensative behaviour.

\begin{proposition}\cite{eusflat25:covering}\label{prop:constru_GCI}
  Let $G:[0,1]^K\to[0,1]$ be a $K$-dimensional grouping function, and let $M:[0,1]^N\to[0,1]$ be a function. The mapping $GCI:\mathcal{M}_{N,K}([0,1])\to [0,1]$ defined as $GCI(\mathbf{U})= (M\circ\mathbf{G})(\mathbf{U}) = M(G(\mathbf{r}_1),\ldots,G(\mathbf{r}_N))$ is a GCI if and only if $M$ is continuous, non-decreasing, symmetric and verifies the following conditions:
  \begin{enumerate}
    \item[$(M1)$] $M(\mathbf{v}) = 0$ if and only if $\mathbf{v} = \mathbf{0}_N$.
    \item[$(M2)$] $M(\mathbf{v}) = 1$ if and only if $\mathbf{v} = \mathbf{1}_N$.
  \end{enumerate}
\end{proposition}

The motivation behind MEGCIs arises from the fact that, although weighted means are an appealing option to play the role of the function $M$ in Prop. \ref{prop:constru_GCI}, this construction method does not work in general in such a case. There are two reasons: First, given a weighting vector $\mathbf{w} = (w_1, \dots, w_N)\in[0,1]^N$, the associated weighted mean is in general a non-symmetric operator. This occurs because the weights $w_i$ are linked to specific positions, that remain fixed even if the objects $x_i$ permute. Consequently, $(M \circ \mathbf{G})$ fails to satisfy external symmetry of rows. Second, weighted means satisfy conditions $(M1)$ and $(M2)$ only when no element of their weighting vector is 0. Otherwise, $(M \circ \mathbf{G})$ does not verify properties $(GCI1)$ and $(GCI2)$.

The first isssue can be addressed by using OWA operators, which, as exposed in Section \ref{sec:prel}, constitute a particular class of symmetric weighted means. Similarly, the second issue can be solved by avoiding null weights. Therefore, the following result is straightforward.

\begin{proposition}\label{prop:OWA-GCI}
  Let $G$ be a $K$-dimensional grouping function, let $\mathbf{w} \in(0,1]^N$ be a weighting vector with strictly positive elements, and let $M_\mathbf{w}$ denote the associated OWA operator. Then, $(M_\mathbf{w}\circ\mathbf{G})$ is a GCI.
\end{proposition}\par\medskip


These OWA-based GCIs offer the interesting capability of assigning weights to rows $\textbf{r}_i$ depending on the magnitude of the gruoping values $G(\textbf{r}_i)$, to then perform aggregation through the corresponding (ordered) weighted mean of those values. For instance, this allows weighting more heavily those rows with a lower grouping, in such a way that the resulting MAO becomes more demanding regarding the fullfilment of the property expressed by $G$. Moreover, as mentioned in Section \ref{sec:prel}, it is possible to quantify such a requirement level by means of the orness degree $
\mathcal{O}(\mathbf{w})$ of the associated weighting vector.

In this sense, it is important to note that, given a requirement level $d \in (0,1)$, there are infinitely many weighting vectors $\textbf{w}\in (0,1]^N$ such that $\mathcal{O}(\textbf{w})=d$. From a practical point of view, this diversity poses a dificulty since it entails the need for devising a specific weighting vector with the desired orness degree, and most importantly, featuring a sound behavior in terms of how the $N$ values $G(\textbf{r}_i)$ are aggregated. 

Maximum entropy weights provide a convenient tool to address this issue, because of the three properties signaled in Section \ref{sec:prel}: First, for any $d \in [0,1]$, there exists a unique maximum-entropy weighting vector $\textbf{w}^d$ such that $
\mathcal{O}(\textbf{w}^d)=d$. Second, if $d \in (0,1)$, then $\textbf{w}^d\in(0,1]^N$, i.e., it contains no null weights. Third, by definition, maximum entropy weights follow the most uniform distribution possible for the given orness. This entails that the corresponding (ME)OWA operator $M_{\textbf{w}^d}$ assigns the most similar importance possible to all rows $\textbf{r}_i$. 

Consequently, these properties together with Props. \ref{prop:constru_GCI} and \ref{prop:OWA-GCI} guarantee that there exists a unique MEOWA-based GCI for each choice of grouping function $G$ and orness degree $d\in(0,1)$. This fact is reflected in the following definition.

\begin{definition}[Maximum Entropy Global Coverage Index (MEGCI)]
  Let $G$ be $K$-dimensional grouping function, and let $d\in(0,1)$. The MAO $S_{G,d}:\mathcal{M}_{N,K}([0,1])\rightarrow[0,1]$ defined as $$S_{G,d}(\mathbf{U})=M_{\textbf{w}^d}(G(\textbf{r}_1),\dots,G(\textbf{r}_N))$$ is referred to as the \textit{maximum entropy global coverage index} (MEGCI) corresponding to orness degree $d$ and grouping function $G$. 
\end{definition}

The family of all MEGCI operators can be understood as the union of the subfamilies $MEGCI(G) = \{S_{G,d}|d\in(0,1)\}$ associated to different grouping functions $G$. Each such subfamily can thus be provided with a linear order using the orness degree $o$ of its elements. However, this total order does neither have minimum nor maximum element, although it is easy to see that the MAOs $S_{G,0}=\lim_{d\rightarrow 0}S_{G,d}=\min\circ\mathbf{G}$ and $S_{G,1}=\lim_{d\rightarrow 1}S_{G,d}=\max\circ\mathbf{G}$ are respectively the infimum and supremum of $MEGCI(G)$. Therefore, there exists elements of $MEGCI(G)$ that are arbitrarily close to either $S_{G,0}$ or $S_{G,1}$. For this reason, and since $S_{G,0}$ or $S_{G,1}$ are easily computable, in practice it may be convenient to consider $S_{G,0}$ and $S_{G,1}$ as the limiting members of $MEGCI(G)$.

The subfamily $MEGCI(\max)$ takes on a special relevance in the mentioned context in which the rows $\textbf{r}_i$ and columns $\textbf{c}_j$ of the matrix $\mathbf{U}$ are respectively linked to objects and classes, where objects $x_i$ are associated with, or covered by, the classes $c_j$ up to a degree $u_{ij}$. In this case, the grouping step $G(\textbf{r}_i)=\max(u_{i1},\ldots,u_{iK})$ implicitly assigns each object $x_i$ to the class it is most associated with. Consequently, the operators $S_{\max,d}\in MEGCI(\max)$ can be understood as an spectrum of quality indexes of the system formed by the classes, the objects, and their degrees of association with the clasess they are assigned to. Particularly, $S_{\max,0.5}$ can be understood as the most balanced index, in a sense the central element of the family. As $d$ tends to 0, the average of the vector $(\max(\mathbf{r}_1),\ldots,\max(\mathbf{r}_N))$ carried out by $M_{\mathbf{w}^d}$ tends to impose an increasing weight on the objects with a lowest coverage, thus forcing those smallest $\max(\mathbf{r}_i)$ to be greater in order to attain a given quality level. In the limit, the quality valuation of the system would be equal to the coverage of the object with the lowest maximum association, i.e., that provided by the index $S_{\max,0} = \min_i(\max_j(u_{ij}))$. Conversely, when $d$ tends to 1, the quality requirements are relaxed, since then the best covered objects tend to be given a greater weight. In the limit, quality would be identified with the maximum association of the object that gets best covered, i.e., $S_{\max,1} = \max_{i,j}(u_{ij})$.

\section{Application: MEGCI-based Cluster Quality Assessment}\label{sec:application}
\noindent This section is devoted to illustrate the utilization of MAOs, and specifically MEGCIs, for clustering quality analysis purposes. To this aim, similarly to Example \ref{ex:fuzzyClustering}, it is assumed that the relevant information, describing the association between $N$ objects $X=\{x_1,\dots,x_N\}\subset \mathbb{R}^p$ to be grouped and $K$ clusters forming a partition $\mathcal{P}_K=\{C_1,\dots,C_K\}$ of those objects, is arranged as an $N\times K$ matrix $\mathbf{U}_K$. The derivation of quality indexes of $\mathcal{P}_K$ from $\mathbf{U}_K$ can thus be understood as an aggregation task to be performed through MAOs.

Then, Section \ref{subsec:exp_cluster_cohesion} demonstrates that, with appropriate choices for the information represented in $\mathbf{U}_K$, MEGCIs provide a family of useful indexes for assessing partition quality. Specifically, two alternative definitions for the elements $u_{ij}$ of $\mathbf{U}_K$ are considered: i) Fuzzy $K$-Means (FCM) membership degrees, and ii) the coverage degrees introduced in \cite{eusflat25:covering,rodriguez2025crb}. Notably, it is observed that only under the latter choice can certain MEGCIs be interpreted as cluster cohesion (or compactness) measures. It is then shown that these MEGCI-based cohesion measures can be easily transformed into internal cluster validity indexes (ICVIs) to assess various quality aspects of a partition. From this basis, Section \ref {subsec:exp_models} describes a more extensive experimental study\footnote{The experimental code of this study is publicly available at \url{https://github.com/asieriko/MAO-ClusterCohesion}.} assessing the application of MEGCI-based ICVIs to the estimation of the optimal number of clusters. 


\subsection{MEGCI-based Internal Cluster Cohesion Measures and Cluster Validity Indexes}\label{subsec:exp_cluster_cohesion}

\noindent Cluster cohesion measures 
quantify the level of compactness or homogeneity of the objects being grouped into a cluster. The notion of cluster cohesion is naturally extended to partitions formed by several clusters, and indeed cohesion measures are mostly used at this level to directly quantify partition cohesion. Traditionally, cohesion has been associated with tightly-grouped objects and thus with low intra-cluster distances or variance. Therefore, in the context of prototype- or centroid-based clustering algorithms such as $K$-Means (KM) or FCM, which implicitly operate under the assumption of spherical, isotropic clusters, cohesion is usually understood to increase when distances $\|x_i - c_j\|$ between objects $x_i$ and their cluster centroid $c_j\in\mathbb{R}^p$ decrease. Due to the greedy nature of these algorithms, this entails that cohesion is expected to increase when the partition's number of clusters $K$ grows.

\begin{figure}[ht]
\centering
\includegraphics[width=\linewidth]{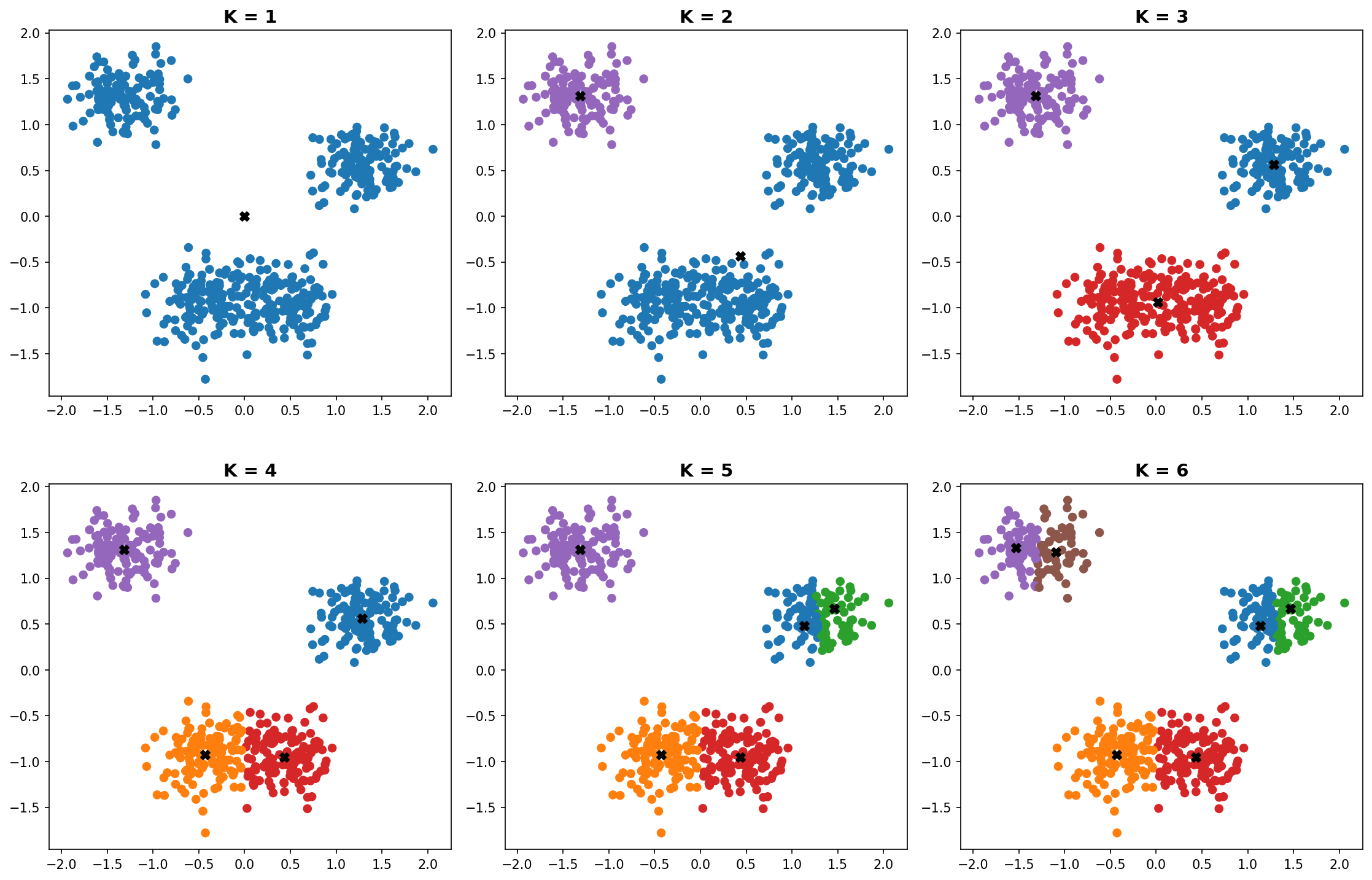}
\caption{Partitions $\mathcal{P}_1,\dots,\mathcal{P}_6$ of an example 2-dimensional dataset containing $N = 1024$ samples balancedly drawn from $K_{true} = 4$ isotropic Gaussian distributions. A scaling factor $\sigma_{scale} = 0.3$ is used, resulting in a moderate overlap of two of the four groups. The partitions were obtained by means of the FCM algorithm, using fuzziness parameter $m=2$ and Euclidean distance, and assigning each sample to the cluster with maximum membership degree \eqref{eq:fcm_degrees}. A black cross identifies each cluster centroid.}
\label{fig:6partitions}
\end{figure}

We illustrate the application of MEGCIs in this context using a synthetic 2-dimensional example dataset with $K_{true}=4$ actual groups. The FCM algorithm is applied on this dataset after its standardization, producing a sequence of partitions $\mathcal{P}_K$, with the number of clusters $K$ ranging from 1 to a certain upper bound $K_{max}\leq N$. Here, $K_{max}=50$ is employed. The first six partitions $\mathcal{P}_1,\dots,\mathcal{P}_6$ are shown in Figure \ref{fig:6partitions}. As mentioned, the elements $u_{ij}$ of the matrices $\mathbf{U}_K$ are computed in two alternative ways:
\begin{enumerate}
  \item FCM membership degrees, defined as 
  \begin{equation}\label{eq:fcm_degrees}
    u_{ij} = \left( \sum_{h=1}^K  \left(\frac{\| x_i - c_j\|}{\| x_i - c_h\|} \right)^{\frac{2}{m-1}}  \right)^{-1}.
  \end{equation}
  It holds that $u_{i1} + \dots + u_{ik} = 1$ for any object $x_i$. The resulting matrix of membership degrees is denoted by $\mathbf{U}_K^{FCM}$.
  
  \item Coverage degrees, which model the association between objects and clusters as exponentially decaying with the distance of the former to the latter's centroids. They are expressed as 
  \begin{equation}\label{eq:coverage_degrees}
    u_{ij} = \exp^{-\frac{r}{s}\| x_i - c_j\|},
  \end{equation}
  \noindent where $r=2\ln10$ is chosen to associate fractional powers of $10^{-2}$ with the corresponding fractions of distance $s$, and in turn $s=5\sqrt{p}$ is employed for standardized data. Contrarily to FCM membership degrees, coverage degrees are intended to measure the association between objects and clusters in an absolute (instead of correlative) way, and thus they  have not to sum up to 1 for each object. The resulting matrix is denoted as $\mathbf{U}_K^{COV}$.
\end{enumerate}
 
Then, the MEGCI operators $S_{\max,0},\ S_{\max,0.5}$ and $S_{\max,1}$ are applied on the sequence of matrices $\mathbf{U}_K^{FCM}$ and $\mathbf{U}_K^{COV}$ such that $K=1,\dots,K_{max}$. As discussed in Section \ref{sec:covering}, these three operators constitute the limiting and central instances of the subfamily of MEGCI operators $S_{G,d}$ based on the grouping function $G=\max$, with orness degrees $d\in\{0,0.5,1\}$. To simplify, the notation $S_d$ shall be used from now on to refer to the operator $S_{\max,d}$. Since the choice of $G=\max$ can be interpreted as an assignment of each object to the cluster with the closest centroid, for each partition $\mathcal{P}_K$ the considered indexes $S_d(\mathbf{U}_K)$ respectively inform of the minimum, mean and maximum across the membership and covering degrees of all objects in the clusters they are assigned to.  

\begin{table}[h]
\centering
\caption{Behavior of the different indices for the sequence of  partitions shown in Figure \ref{fig:monotonicity}.}
\label{tab:monotonicity}
\begin{tabular}{lcccccc}
\hline
Index & $K=1$ & $K=2$ & $K=3$ & $K=4$ & $K=5$ & $K=6$ \\
\hline
$S_0(\mathbf{U}_K^{FCM})$ & 1.0000 & 0.6260 & 0.5330 & 0.4765 & 0.4431 & 0.4250\\
$S_{0.5}(\mathbf{U}_K^{FCM})$ & 1.0000 & 0.8868 & 0.9283 & 0.8918 & 0.8311 & 0.7837\\
$S_1(\mathbf{U}_K^{FCM})$ & 1.0000 & 0.9999 & 0.9999 & 0.9999 & 0.9999 & 0.9990 \\
$S_0(\mathbf{U}_K^{COV})$ & 0.2132 & 0.2724 & 0.4880 & 0.5742 & 0.5742 & 0.5742 \\
$S_{0.5}(\mathbf{U}_K^{COV})$ & 0.4290 & 0.6171 & 0.7795 & 0.8260 & 0.8334 & 0.8432 \\
$S_1(\mathbf{U}_K^{COV})$ & 0.7375 & 0.9833 & 0.9853 & 0.9907 & 0.9889 & 0.9853 \\
$XB_K$ & - & 0.1265 & 0.0462 & 0.1196 & 0.6639 & 0.5419\\
$CV_K^{0.5}$ & - & 1.1583 &	3.4925 & 6.2838 &	0.7551 & 0.8048 \\

\hline
\end{tabular}
\end{table}

\begin{figure}[ht]
  \centering
  \includegraphics[width=1\textwidth]{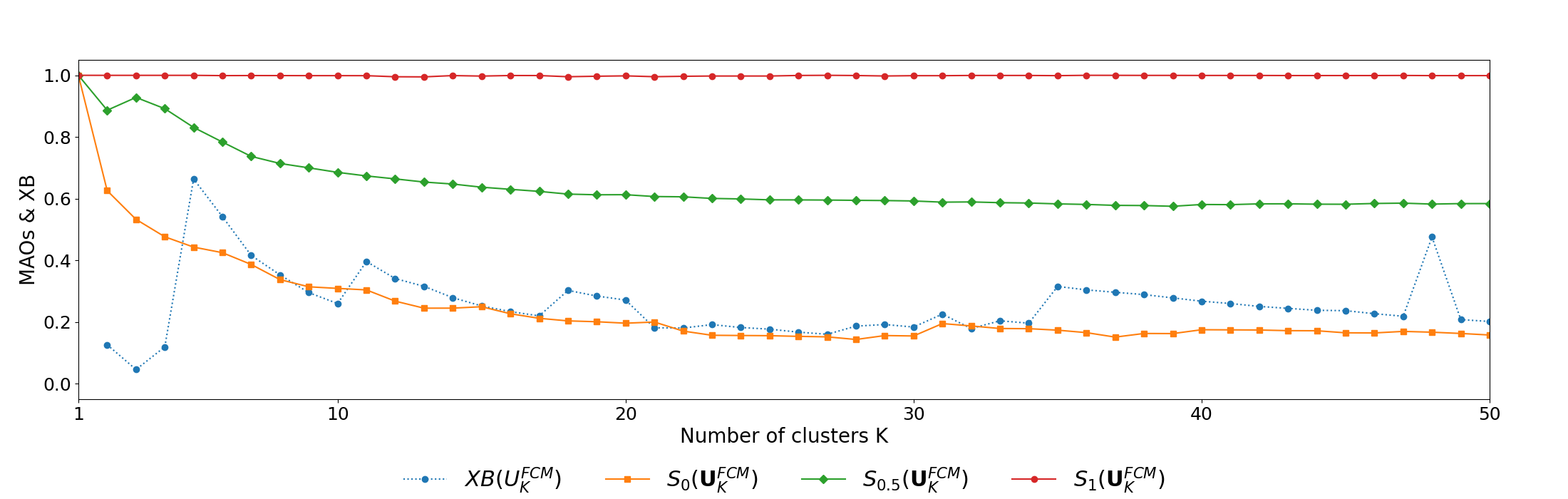}
  \includegraphics[width=1\textwidth]{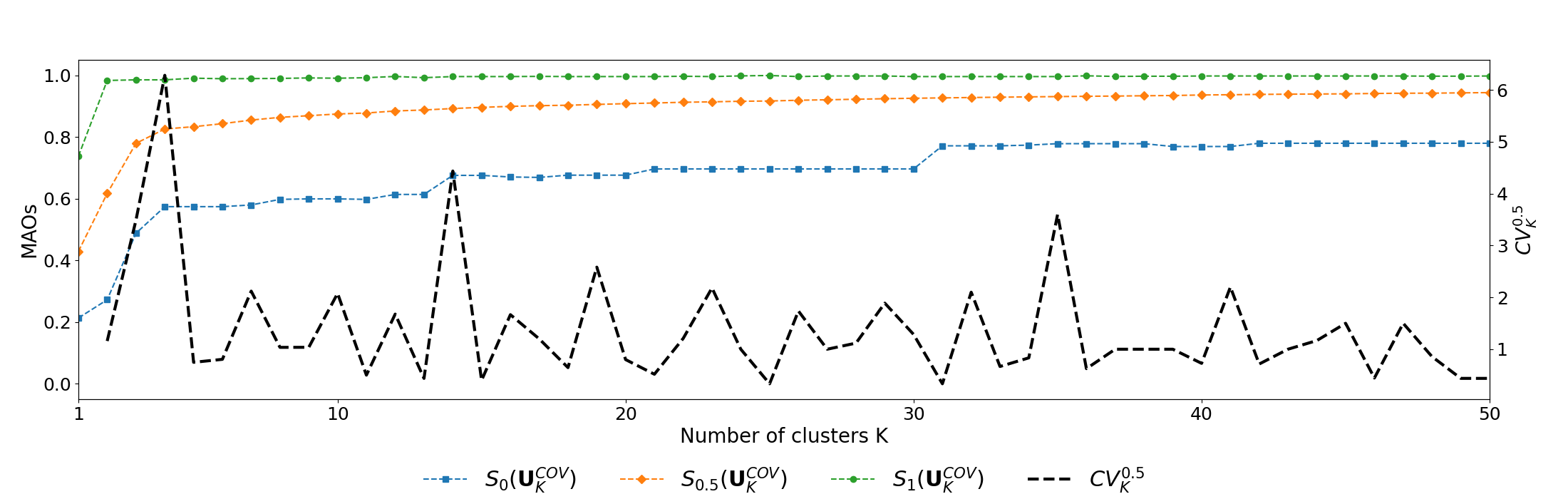}
  \caption{Value of each index for partitions  $\mathcal{P}_K$, $K =1,\dots, 50$. The top panel shows the MEGCIs obtained from the $\mathbf{U}_K^{FCM}$ matrices of membership values returned by FCM, as well as the corresponding XB index. The bottom panel shows the MEGCIs obtained from the $\mathbf{U}_K^{COV}$ matrices of coverage degrees, together with the curvature index corresponding to the cohesion measure $S_{0.5}(\mathbf{U}_K^{COV}$).}
  \label{fig:monotonicity}
\end{figure}

Figure \ref{fig:monotonicity} shows the six resulting index sequences $S_d(\mathbf{U}_K)$, $K=1,\dots,K_{max}$, while Table \ref{tab:monotonicity} informs of their values for the partitions in Figure \ref{fig:6partitions}. Then, notice first that, despite the matrix it is applied on, the operator $S_1$ leads to excessively stable, non-informative indexes, as it only takes into account the object that is closest to any centroid. This distance tends to be close to 0, and thus both $S_1(\mathbf{U}_K^{FCM})$ and $S_1(\mathbf{U}_K^{COV})$ tend to be approximately 1 for any $K$. In general, operators $S_d$ with $d\approx1$ will exhibit this insentive behaviour, rendering them unuseful as cohesion measures. Second, $S_0(\mathbf{U}_K^{FCM})$ and $S_{0.5}(\mathbf{U}_K^{FCM})$ present a counterintuitive behavior, as they tend to decrease or to fluctuate when the partitions number of clusters increase, that is, when distances between objects and centroids decrease, and cohesion is thus expected to rise\footnote{Since MAOs are non-decreasing functions, this can only happen because of the non-monotonic and/or decreasing behavior of FCM membership degrees with respect to the number of clusters $K$. In turn, as can be seen from Eq. \eqref{eq:fcm_degrees}, this behavior is due to the probabilistic definition of these degrees, which assess the association between an object and a cluster relatively to that of the same object with other clusters, instead of as an absolute quantity, independent of other association degrees. Indeed, when $K=1$ all objects are assigned with maximum membership degree 1 to that single cluster, regardless of their situation with respect to the global centroid.}. 

In contrast, when $S_0$ and $S_{0.5}$ are applied on $\mathbf{U}^{COV}$, the resulting indexes do present a behavior much more compatible with that of a cohesion measure, especially $S_{0.5}$: it exhibits a sensitive monotonically increasing trend, smoothly approaching 1 as $K$ increases, and without locally stabilizing in plateaus, as $S_0$ instead does. However, although this last may not provide a sensitive cohesion measure, it is indeed a useful index, as it informs of the coverage degree of the worst covered object, which can be a quite relevant quantity in certain contexts to assess partition quality. In general, operators $S_d$ with a relatively low orness $d$ assess cohesion mainly through a minority of not well-covered objects, while those with a medium $d$ provide a more holistic or global picture of cohesion. Moreover, the range between indexes $S_0(\mathbf{U}^{COV})$ and $S_0(\mathbf{U}^{COV})$ also behaves as expected, reducing the gap between the least and most covered objects as $K$ increases. It can as well be seen that the mean coverage $S_{0.5}(\mathbf{U}_k^{COV})$ seems to stabilize at the correct number of clusters, i.e. when $K=K_{true}=4$.

This last observation suggests that those MEGCIs behaving as cohesion measures, such as $S_{0.5}(\mathbf{U}^{COV})$, could carry relevant information for estimating the right number of clusters $K_{true}$. However, due to their monotonic behavior these MEGCIs cannot be directly employed as standard ICVIs by simply looking at their global optimum, the maximum in this case. Instead, to discriminate the best partition, an additional processing step is required. In this work, this is explored by applying the curvature (CV) method proposed in \cite{zhang2017curvature} over the sequences $S_d(\mathbf{U}_K^{COV})$, $K=1,\dots,K_{max}$, for different orness degrees $d\in(0,1)$. Specifically, the corresponding sequences of the CV internal clustering validation index are computed as

\begin{equation} CV^d_K = \frac{S_d(\mathbf{U}_K^{COV}) - S_d(\mathbf{U}_{k-1}^{COV})}{S_d(\mathbf{U}_{K+1}^{COV}) - S_d(\mathbf{U}_K^{COV})},
\end{equation}

\noindent for $K$ between 2 and $K_{max} - 1$. The selected number of clusters would then be given by the value of $k$ that maximizes the ratio $CV^d_K$, since it would identify the point where the marginal coverage gain of adding a new cluster is lowest (see \cite{zhang2017curvature}). The application of the CV index in the example dataset above is also described in Figure \ref{fig:monotonicity} and Table \ref{tab:monotonicity}, using $d=0.5$. Particularly, $CV^{0.5}_K$ is maximum for $K=4$, correctly identifying the right number of groups $K_{true}$ of the dataset despite two of them presenting a moderate overlap. In contrast, the well-known XB index in Example \ref{ex:fuzzyClustering} presents a minimum at $K=3$, and thus it is in this case unable of recognizing the actual $K_{true}$. This illustrates that MEGCI-based cohesion measures and ICVIs may improve on existing indexes at assessing some clustering quality aspects.

\subsection{Optimal cluster number estimation}\label{subsec:exp_models}

\noindent This section presents a more systematic analysis of the above illustrated number of clusters estimation (NCE) capability of MEGCI-based ICVIs. To this aim, an extensive computational study is carried out, and their performance at this task is compared to that of a baseline of traditional ICVIs.

\subsubsection{Experimental Setup}\label{subsec:exp_data}

\noindent To provide a robust evaluation across a wide variety of cluster complexity conditions, a comprehensive battery of synthetic isotropic Gaussian datasets was generated following the methodology established in \cite{rodriguez2025crb, masud2018nice}. Table \ref{tab:synthetic_params} shows the considered experimental grid, which yields a total of 400 evaluation datasets. 

\begin{table}[h]
\centering
\caption{Parameters used for the generation of the synthetic datasets.}
\label{tab:synthetic_params}
\begin{tabular}{lll}
\hline
\textbf{Parameter} & \textbf{Symbol} & \textbf{ Values} \\
\hline
Number of Groups & $K_{true}$ & $\{2, 4, 8, 16, 32\}$ \\
Dimensionality & $p$ & $\{2, 5, 10, 15\}$ \\
Number of Samples & $N$ & $\{1024, 2048, 4096, 8192\}$ \\
Std. Deviation & $\sigma_{scale}$ & $0.10, 0.15, 0.20, 0.25, 0.30$ \\
\hline
\end{tabular}
\end{table}

Then, the KM and FCM algorithms are applied on the standarized datasets, using the default scikit-learn and scikit-fuzzy implementations, respectively. For each algorithm and dataset, a sequence of partitions $\mathcal{P}_K$ is obtained, with $K=1,\ldots,K_{max}=50$. The optimal number of clusters of the dataset is then estimated using both the MEGCI-based $CV^d$ index introduced above, and a baseline of well-known ICVIs. Specifically, the indexes composing the baseline are:  Partition Coefficient (PC, \cite{Wu2005}), Xie-Beni (XB, \cite{xieBeni1991validity}), Davies-Bouldin (DB, \cite{davies1979cluster}), Silhouette Score (SS, \cite{ROUSSEEUW198753}), and Calinski-Harabasz (CH, \cite{calinski1974dendrite}). 
The PC and XB indexes are obtained from the fuzzy partitions provided by FCM, while DB, SS, CH and the hard (i.e., crisp) version of XB, denoted as XBH, use the KM partitions. 
Regarding the $CV^d$ index, it is computed from the coverage degrees in Eq. \eqref{eq:coverage_degrees}, and these in turn have been computed in two ways, using either the crisp partitions derived from FCM through maximum membership asignment or the KM ones. Moreover, to explore the effect of the orness degree $d$ on NCE performance, a total of 13 values of this parameter are considered. This leads to a spectrum of $CV^d$ indexes ranging from the purely conjunctive $CV^{0}$ to the neutral $CV^{0.5}$ (including orness degrees from 0.05 to 0.45 in steps of 0.1), and up to the purely disjunctive $CV^{1}$ (likewise including degrees from 0.55 to 0.95 in steps of 0.1).

\subsubsection{Experimental results}

\noindent The performance of the considered baseline and MEGCI-based ICVIs across the 400 synthetic datasets is summarized in Table \ref{tab:result_number_k}, and visually depicted in Figure \ref{fig:result_number_k}. The results detail both the absolute number of successful estimations (Hits, i.e., the number of times each method correctly identified the actual $K_{true}$) and the corresponding overall accuracy percentage (Acc.). Although the 13 $CV^d$ indexes were computed for both FCM and KM partitions, only the KM results are reported as they consistently outperformed FCM in accuracy across all indexes.

\begin{table}[ht]
\centering
\caption{Performance at the NCE task of baseline and MEGCI-based ICVIs across the 400 synthetic datasets.}
\label{tab:result_number_k}
\begin{tabular}{l c c l c c}
\hline
\textbf{Methods} & \textbf{Hits} & \textbf{Acc. (\%)} & \textbf{Methods} & \textbf{Hits} & \textbf{Acc. (\%)}\\
\hline
PC & 147 & 36.75 & $CV^{0.05}$ & 217 & 54.25 \\
XB & 154 & 38.50 & $CV^{0.15}$ & 255 & 63.75 \\
XBH &162 & 40.50 & $CV^{0.25}$ & 267 & 66.75\\
DB & 191 & 47.75  & $CV^{0.35}$ & 280 & 70.00  \\
SS & 230 & 57.50  & $CV^{0.45}$ & 293 & 73.25 \\
CH & 249 & 62.25  & $CV^{0.55}$ & 297 & \textbf{74.25} \\
   &   &   & $CV^{0.65}$ & 285 & 71.25 \\
$CV^0$  &  139 & 34.75  & $CV^{0.75}$ & 250 & 62.50 \\
$CV^{0.5}$ & 296  & 74.00  & $CV^{0.85}$ & 189 & 47.25 \\
$CV^1$  &  56 &  14.00 & $CV^{0.95}$ & 114 & 28.50 \\
\hline
\end{tabular}
\end{table}

\begin{figure}[ht]
  \centering
  \includegraphics[width=1\textwidth]{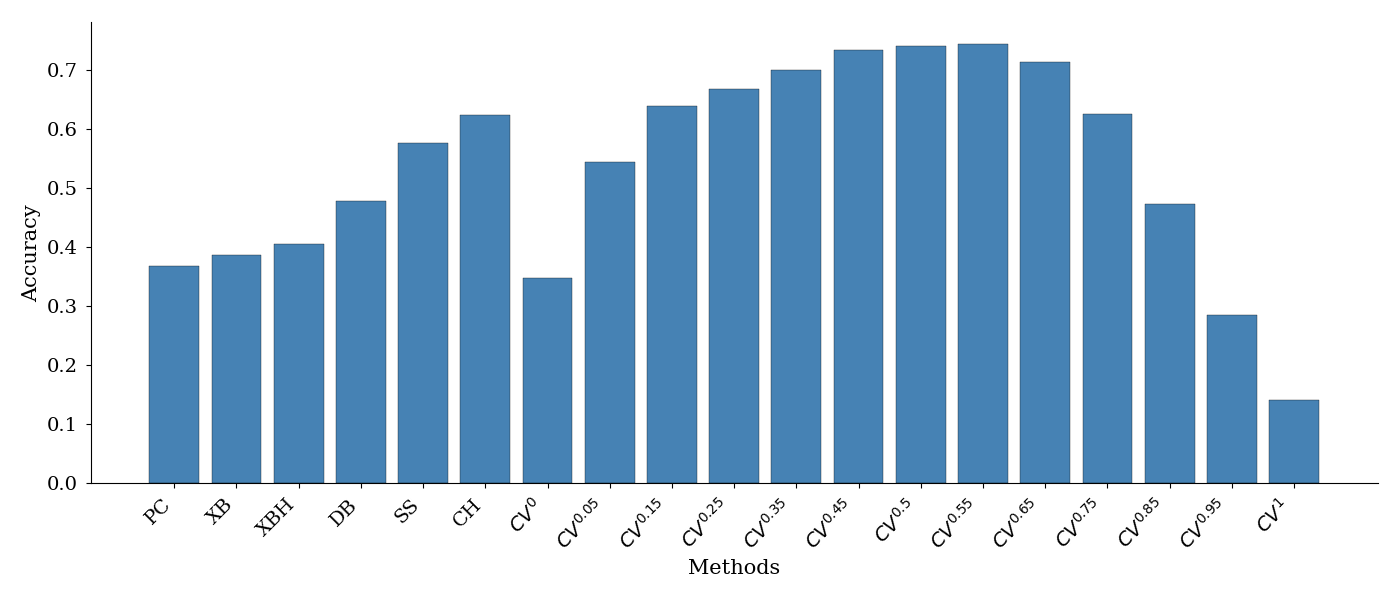}
  \caption{Proportion of correct identifications of the actual number of groups $K$ across the considered 400 datasets, by method.}
  \label{fig:result_number_k}
\end{figure}

Notably, a broad range of the proposed MEGCI-based $CV^d$ indexes, particularly those with $d\in[0.15,0.75]$, attains a better accuracy than that of all baseline ICVIs. Within this band, the estimation accuracy is relatively stable, ranging from 62.50\% up to 74.25\%, and reaching its peak for $CV^{0.55}. $Moreover, there seems to be a significant difference, consistently greater than 10\% in absolute terms, between the 62.25\% accuracy of the best-performing baseline method, CH, and that of the $CV^d$ indexes with medium orness degree, say $d\in[0.45,0.55]$. In contrast, extreme indexes such as the purely conjunctive $CV^{0}$ (34.75\%) or the purely disjunctive $CV^{1}$ (14.00\%) offer a much lower performance. Finally, although in this study $CV^{0.55}$ achieves the best performance, it is important to point out that indexes with a more conjunctive behaviour, i.e., those with $d\in[0,0.45]$, exhibit a significantly better mean performance (60.46\% Acc.) than those with a more disjunctive behavior (mean Acc. across $d\in[0.55,1]$ is 49.63\%).  

Therefore, these commpetitive results of the $CV^d$ indexes assert the flexibility and usefulness of MEGCI-based cohesion measures and ICVIs built upon coverage degrees to evaluate clustering quality aspects, such as the optimal number of clusters.

\section{Conclusions and further research}\label{sec:conclusion}

The need to summarize a collection of membership degrees into a single representative value is one of the hottest topic in fuzzy logic and soft computing. Whenever this information involves not one but two distinct dimensions -- for instance, when a set of $N$ objects is jointly evaluated with respect to $K$ classes, criteria, or decision-makers -- the resulting data is naturally represented by an $N\times K$ membership matrix $\mathbf{U}$ rather than a vector. As discussed throughout this paper, reducing such a matrix to a vector, although mathematically always possible, may discard semantically meaningful relationships between rows and columns that should instead be preserved during the aggregation process.

Motivated by this observation, this paper has introduced and formalized the notion of \emph{matrix aggregation operator} (MAO), a mapping $F:\mathcal{M}_{N,K}([0,1])\to[0,1]$ satisfying the natural boundary and monotonicity conditions of an aggregation operator, but acting directly on matrices instead of vectors, thereby preserving the original bi-dimensional structure of the collection of observations to be aggregated. Within this framework, particular attention has been paid to \emph{decomposable} MAOs, i.e., those that can be built from a two-stage procedure aggregating first rows (or columns) and then the resulting vector. While decomposability provides a simple and widely applicable method for constructing MAOs from well-established vectorial aggregation operators, it has been shown that it constitutes a proper subclass of all MAOs: explicit examples of non-decomposable operators have been provided, including the polarization index of \cite{guevara20} and the image-based examples discussed in Section~\ref{sec:MAOs}, which combine matrix entries in ways that cannot be reduced to independent row- or column-wise processing.

Following this natural matrix structure, a second main contribution concerns the study of symmetry properties specific to the matrix setting. Beyond the standard (global) symmetry inherited from vectorial aggregation operators, the two-dimensional structure of $\mathbf{U}$ allows for a richer taxonomy of invariance properties, arising from permutations of rows, columns, or their internal elements. Seven non-equivalent row- and column-symmetry conditions have been identified and organized according to the logical relationships that hold between them, providing a systematic account of how symmetry and matrix structure interact. In particular, it has been shown that requiring the strongest of these properties simultaneously for both rows and columns collapses a MAO into an ordinary vectorial aggregation operator, effectively erasing the role of the matrix structure -- a result that clarifies the extent to which such structure can be meaningfully exploited.

Building on this theoretical framework, a specific and practically relevant family of MAOs, the \emph{maximum entropy global coverage indices} (MEGCIs), has been introduced. MEGCIs generalize the notion of global coverage index by combining a grouping-based row aggregation with an OWA operator whose weights are fixed, for each orness degree, through the maximum entropy principle. This construction guarantees both the fulfillment of the required boundary conditions and a well-defined, symmetric behaviour across the whole spectrum of orness degrees. The usefulness of MEGCIs has been illustrated in the context of clustering quality assessment, where they have been shown to yield meaningful cluster cohesion measures when applied to coverage-based association degrees, in contrast to their less informative behaviour when applied to standard fuzzy $c$-means memberships. The associated MEGCI-based curvature indexes were further evaluated over an extensive battery of $400$ synthetic datasets for the task of estimating the correct number of clusters, outperforming a representative baseline of classical internal cluster validity indexes across a broad range of orness degrees, and thus providing empirical support for the practical value of the proposed matrix aggregation framework.

Overall, this work shows that endowing aggregation theory with an explicit matrix structure is not a purely formal exercise: it yields new operators, uncovers symmetry phenomena that are invisible at the vector level, and translates into tangible gains in a concrete application domain. Several directions for future research naturally follow from this study as the extension the decomposability and symmetry analysis to other properties (e.g., associativity, idempotency), exploring further non-decomposable MAOs for applications such as image processing or network analysis, and generalizing MAOs to \emph{tensor aggregation operators} for data with more than two dimensions.

\vspace{10pt}

\noindent \textbf{Credits}: \textit{Conceptualization} D.G., J.T.R, I.G, J.M., H.B. \textit{Data curation} A. U-L. \textit{Formal analysis} I.G., J.T.R., D.G. \textit{Funding acquisition} D.G., J.T.R, H.B \textit{Investigation} I.G., A. U-L., J.T.R. \textit{Methodology} D.G., J.T.R., A. U-L. \textit{Project administration} D.G., J.T.R., H.B. \textit{Software} A.U-L., J.T.R. \textit{Supervision} J.T.R., D.G.  \textit{Visualization} A.U-L. \textit{Writing – original draft} J.T.R., I.G., D.G., A.U-L. \textit{Writing – review and editing} J.M., H.B.


\bibliography{References}



\clearpage

\appendix
\section{Appendix. Proof of existence of non-decomposable MAOs}\label{appendix}

\noindent Let $F:\mathcal{M}_{2,2}([0,1]) \longrightarrow [0,1]$ be given by
$$
F(\textbf{U}) = \min(u_{11}\cdot u_{22} + u_{21} \cdot u_{12},1)
$$
\noindent for any $\textbf{U}=\big(\begin{smallmatrix}
  u_{11} & u_{12}\\
  u_{21} & u_{22}
\end{smallmatrix}\big) \in \mathcal{M}_{2,2}([0,1])$.

Clearly, $F$ is a MAO since $F(\textbf{0}_{2,2})=\min(0\cdot 0 + 0 \cdot 0,1)=0$, $F(\textbf{1}_{2,2})=\min(1\cdot 1 + 1 \cdot 1,1)=1$, and $F(\textbf{U}) = \min(u_{11}\cdot u_{22} + u_{21} \cdot u_{12},1) \leq \min(u'_{11}\cdot u'_{22} + u'_{21} \cdot u'_{12},1) = F(\textbf{U}')$ whenever $\textbf{U}=\big(\begin{smallmatrix}
  u_{11} & u_{12}\\
  u_{21} & u_{22}
\end{smallmatrix}\big) \leq \big(\begin{smallmatrix}
  u'_{11} & u'_{12}\\
  u'_{21} & u'_{22}
\end{smallmatrix}\big) = \textbf{U}'$.

Now, suppose that $F$ is a row-decomposable MAO, i.e., there exist AOs $C,R_1,R_2:[0,1]^2 \longrightarrow [0,1]$
such that $F(\textbf{U})=C(R_1(u_{11},u_{12}),R_2(u_{21},u_{22}))$, and denote $a=R_1(1,0),\ b=R_1(0,1),\ c=R_2(1,0)$ and $d=R_2(0,1)$. Notice that it has to be $C(b,c)=C(a,d)=1$ and $C(b,d)=C(a,c)=0$. Then, consider the following cases:
\begin{itemize}
    \item If $c=d$, it has to be $0=C(a,c)=C(a,d) = 1$, which is impossible.
    \item If $c<d$, the monotonicity of $C$ entails that $1=C(b,c) \leq C(b,d) = 0$, which is impossible.
    \item If $d<c$, the monotonicity of $C$ again entails that $1=C(a,d) \leq C(a,c) = 0$, which is impossible.
\end{itemize} 

\noindent Hence, it follows that the numbers $c,d\in\mathbb{R}$ have to be incomparable, but this is an absurd since the real line is a total order. Then, the MAO $F$ can not be row-decomposable. A similar argument shows that $F$ can neither be column-decomposable, and thus it is concluded that $F$ is a non-decomposable MAO.

\end{document}